\PassOptionsToPackage{table}{xcolor}
\documentclass[runningheads]{llncs}

\usepackage{eccv}

\usepackage{eccvabbrv}

\usepackage{tikz}
\usetikzlibrary{arrows.meta,calc}
\definecolor{oiBlue}{HTML}{0072B2}
\definecolor{oiVermillion}{HTML}{D55E00}
\definecolor{oiGreen}{HTML}{009E73}
\definecolor{oiBrown}{HTML}{8C510A}
\definecolor{mBlue}{HTML}{3B6EA8}
\definecolor{mRed}{HTML}{B24A3A}
\definecolor{mGreen}{HTML}{3A7F5C}
\definecolor{tfBlue}{HTML}{4C72B0}      
\definecolor{tfOrange}{HTML}{DD8452}    
\definecolor{tfTeal}{HTML}{55A868}      
\definecolor{mBlueEmph}{HTML}{24476B}
\definecolor{oiVermillionEmph}{HTML}{A84700}

\colorlet{dmLine}{cyan!70!teal}
\colorlet{fmLine}{magenta!65!black}

\usepackage{amsthm}

\theoremstyle{definition}

\theoremstyle{definition}
\newtheorem{myproperty}{Property}

\newtheoremstyle{boldremark}
  {3pt}{3pt}            
  {\normalfont}         
  {}                    
  {\bfseries}           
  {.}                   
  { }                   
  {}                    

\theoremstyle{boldremark}
\newtheorem{limitation}[]{Limitation}

\usepackage[linesnumbered,ruled,vlined]{algorithm2e}
\usepackage{algorithmic}
\usepackage{amsmath}
\usepackage{adjustbox}
\usepackage{graphicx}
\usepackage{tabularx}
\usepackage{subcaption}  
\usepackage{wrapfig}
\usepackage{booktabs}
\usepackage{sg-macros}
\usepackage{chngcntr}
\usepackage{multirow}
\usepackage{caption}
\usepackage{bm}
\usepackage{enumitem}
\usepackage{pifont}
\usepackage[accsupp]{axessibility}  

\newcommand{\bmx}{\boldsymbol{x}}
\newcommand{\bmy}{\boldsymbol{y}}
\newcommand{\bmz}{\boldsymbol{z}}
\newcommand{\bmh}{\boldsymbol{h}}

\newcommand{\bmv}{\boldsymbol{v}}

\newcommand{\cmark}{\textcolor{green!60!black}{\ding{51}}}
\newcommand{\xmark}{\textcolor{red}{\ding{55}}}

\newcommand{\KL}[2]{\text{KL}\left(#1 \middle\| #2\right)}

\newcommand\mypara[1]{\noindent\textbf{#1}}

\usepackage{etoc}

\usepackage{hyperref}

\usepackage{orcidlink}

\begin{document}

\title{Straight-Path Flow Matching for Incomplete Multi-View Clustering} 


\newcommand{\equalcontrib}{\textsuperscript{*}}
\newcommand{\corresponding}{\textsuperscript{\(\dagger\)}}

\author{Yiteng Yuan\inst{1}\equalcontrib \and
Junyan Wang\inst{2}\equalcontrib\corresponding \and
Zheyuan Liu\inst{2} \and
Hong Jia\inst{3} \and
Lei Fan\inst{4} \and\\
Zhulin Tao\inst{5}\corresponding \and
Lianbo Guo\inst{1}}

\authorrunning{Y.~Yuan \& J. Wang et al.}


\institute{School of Software Engineering, Huazhong University of Science and Technology, Wuhan, China \and
Australian Institute for Machine Learning, Adelaide University, Adelaide, Australia \and
University of Auckland, Auckland, New Zealand \and
University of New South Wales, Sydney, Australia \and
Communication University of China, Beijing, China\\
\email{yuanyiteng\&lbguo@hust.edu.cn},~\email{junyan.wang\&zheyuan.liu@adelaide.edu.au} \\
\email{hong.jia@auckland.ac.nz},~\email{lei.fan1@unsw.edu.au},~\email{taozl@cuc.edu.cn}}

\maketitle
\renewcommand{\thefootnote}{}
\footnotetext{
\(\star\) Equal contribution. 
\(\dagger\) Corresponding author.
}
\renewcommand{\thefootnote}{\arabic{footnote}}

\begin{abstract}
  Incomplete Multi-View Clustering addresses the problem of clustering multi-modal data when certain views are missing.
  Recent end-to-end generative approaches leverage diffusion models to recover missing views via stochastic noise-to-data trajectories.
  While expressive, such mechanisms are not explicitly designed for clustering, as they initialize from cluster-agnostic noise and rely on stochastic denoising dynamics.
  In this work, we revisit probability path design in end-to-end generative IMVC. We introduce a flow-matching framework with a linear interpolation path between paired view representations, that replaces diffusion with probability flows between observed and missing views.
  We provide a formal analysis showing that deterministic ODE flows are inherently better aligned with clustering objectives than diffusion-based stochastic trajectories, especially in terms of transport mechanisms that respect class-conditional data distributions and maintain cluster consistency in finite-step regimes.
  Building upon this insight, we develop an end-to-end IMVC architecture that integrates straight-path flow-matching view completion with cluster-level and entropy-based alignment to enforce cross-view clustering consistency.
  Extensive experiments on standard IMVC benchmarks demonstrate that the proposed framework establishes new state-of-the-art performance.
  
\end{abstract}

\section{Introduction}
\label{sec:intro}

The task of Incomplete Multi-View Clustering (IMVC)~\cite{gao2016incomplete} addresses the challenge of multi-view clustering (MVC)~\cite{bickel2004multi}, with an additional layer of difficulty, \ie not all views are consistently available for each sample.
Compared to the traditional MVC task setup, IMVC more closely reflects real-world use cases involving multi-view data~\cite{wen2024diffusion,zhang2023robust,zhang2025incomplete}, as such data collection processes may encounter sensor malfunctions, occlusions, disk space limitations, or even storage media corruptions.

Traditional IMVC methods~\cite{yin2021incomplete,liu2020efficient,zhao2017multi,wen2023graph,liu2019efficient,wen2018incomplete} focus on latent alignment and reconstruction under missing-view constraints.
Recently, generative modeling~\cite{ho2020denoising,songdenoising,dhariwal2021diffusion} has emerged as a powerful alternative paradigm.
Diffusion-based approaches~\cite{wen2024diffusion,zhang2025incomplete}, in particular, formulate missing-view recovery as conditional generation: starting from Gaussian noise, a stochastic denoising process conditioned on the observed view progressively reconstructs the missing representation.
Earlier works~\cite{wen2024diffusion} adopt a two-stage pipeline in which clustering is performed after generative recovery.
Recent research has shifted toward end-to-end IMVC frameworks~\cite{zhang2025incomplete} that optimize missing-view recovery and clustering jointly.
In the joint framework, the generative module is trained under clustering-oriented supervision, so the intermediate and final recovered features directly determine the clustering embeddings and assignments during training. Consequently, the properties of the generative trajectory directly influence the structural organization of the learned representations, as cross-view discrepancies are progressively reduced and intra-cluster coherence is reinforced under clustering-oriented supervision.

However, diffusion-based generative modeling typically relies on stochastic dynamics, often formulated as stochastic differential equations (SDEs)~\cite{song2020score}, that transform cluster-agnostic Gaussian noise into data representations through progressive denoising.
As illustrated in~\figref{fig:fm_vs_diffusion_paths}, the reverse process starts from noise that is independent of the observed sample and gradually reconstructs the target representation, causing trajectories to traverse regions of the latent space that are initially unrelated to the underlying cluster structure. Moreover, during the early and intermediate stages of denoising, the representations remain strongly influenced by noise, and cluster-discriminative structure becomes evident only in the low-noise regime (\figref{fig:diffusion-seperate}).
Although such properties are natural for general-purpose generation, IMVC instead requires view-completion mechanisms that are deterministically conditioned on observed representations and preserve the clustering structure of the representations throughout the transport process.

To this end, we adopt flow matching~\cite{lipman2022flow} to learn a deterministic ODE transport operator for missing-view recovery.
Unlike stochastic diffusion trajectories, flow matching enables direct modeling of continuous transports anchored at the observed representation and consistent with the clustering structure of the representation space.
Concretely, we specify a linear interpolation (\ie straight-path) between paired latent representations of observed and missing views, and train a neural vector field to match the corresponding path velocity.
This yields trajectories that evolve consistently with the underlying cluster structure while requiring only a small number of integration steps.
A formal analysis comparing diffusion trajectories and flow-based transport is provided to clarify their structural differences.
Building on this formulation, we integrate straight-path completion with an end-to-end IMVC framework that combines within-view reconstruction and cross-view alignment, where the cluster-level contrastive loss enforces assignment consistency and the entropy-based alignment stabilizes probability predictions.
The resulting design jointly optimizes view recovery and clustering consistency within a unified framework.

Our contributions are summarized as follows:
\begin{itemize}
    \item We propose a flow-matching framework for incomplete multi-view clustering (IMVC) that performs cross-view completion by learning a vector field to transport an observed latent representation to its missing-view counterpart along a straight-path formulation.
    \item We present a formal analysis showing straight-path transport and finite-step cluster separability of flow matching under standard assumptions.
    \item We develop an end-to-end model that integrates within-view reconstruction, cross-view alignment (cluster-level and entropic alignment), and straight-path completion, achieving SOTA performance on IMVC benchmarks.
\end{itemize}

\section{Related Work}
\label{sec:related-work}

\mypara{Deep Incomplete Multi-View Clustering.}
Handling missing modalities is a central challenge in incomplete multi-view clustering (IMVC).
Autoencoder-based approaches~\cite{zhang2019cpm,wei2020deep,xu2022deep} learn latent representations by reconstructing observed views.
Graph-based methods~\cite{wen2021structural,wang2022incomplete,chao2024incomplete} propagate information across samples through graph structures. 
Contrastive learning approaches~\cite{lin2021completer,feng2024partial,li2023incomplete} improve cross-view representation consistency by maximizing agreement between views.
To explicitly reconstruct missing modalities, recent studies introduce generative models~\cite{gui2021review}. 
Generative Adversarial Networks (GANs)~\cite{wang2020generative,lin2023consistent,wang2023self} enable view synthesis but often suffer from training instability and mode collapse. 
Variational Autoencoders (VAEs)~\cite{kingma2013auto,cai2023realize} provide a probabilistic framework for view generation.
More recently, diffusion models have achieved remarkable progress in generative modeling~\cite{ho2020denoising,podell2023sdxl,labs2025flux}.
They have also been explored for missing-view generation in IMVC~\cite{wen2024diffusion,zhang2025incomplete}.
However, diffusion models generate representations through stochastic noise-to-data trajectories that start from cluster-agnostic noise, so the transport process does not explicitly preserve the clustering structure of the representations.

\mypara{Flow Matching in Generative Modeling.}
Flow Matching (FM)~\cite{lipman2022flow,liu2022flow,klein2023equivariant} is a generative framework that learns deterministic vector fields to transport samples between probability distributions through ODE flows. 
Instead of relying on stochastic noise injection and iterative denoising as in diffusion models, FM directly models probability transport along predefined interpolation paths between data distributions. 
Several extensions have further expanded its applicability, including latent-space flow matching~\cite{dao2023flow}, flow generator matching~\cite{huang2024flow}, and hybrid diffusion–flow formulations~\cite{schusterbauer2025diff2flow}. 
The deterministic transport formulation of flow matching provides a natural mechanism for modeling representation transport between paired observations. 
Motivated by this formulation, we adopt flow matching to model cross-view completion in the latent representation space for incomplete multi-view clustering.

\section{Methodology}
\label{sec:method}

In this section, we first concretely demonstrate that, under three assumptions, flow matching exhibits desirable theoretical properties over conditional diffusion models when applied to the task of IMVC (\secref{sec:theoretical-analysis}).
Based on the theoretical analysis, we proceed to propose our method of adopting flow matching for the task in~\secref{sec:flow-based-method}.

\subsection{Motivation of Adopting Flow Matching}\label{sec:theoretical-analysis}

\subsubsection{Problem Setup.}\label{sec:problem-setup-short}

Without loss of generality, we consider a two-view IMVC setting.
Let $\mathcal{X} \subset \mathbb{R}^D$ denote the original data space and $\mathcal{Z} \subset \mathbb{R}^d$ the corresponding latent representation space.
For each sample, the raw observations from the two views are denoted by $\boldsymbol{x}_1, \boldsymbol{x}_2 \in \mathcal{X}$, and their encoded latent representations are $\boldsymbol{z}_1, \boldsymbol{z}_2 \in \mathcal{Z}$, respectively.

Under this formulation, the core objective of IMVC is as follows: given an observed view representation (\eg $\bmz_1$), recover the missing-view representation (\eg $\bmz_2$) such that both representations correspond to the same cluster.
Equivalently, if $c: \mathcal{Z} \to \{1,\dots,K\}$ denotes the clustering assignment function, the desired consistency condition is $c(\boldsymbol{z}_1) = c(\boldsymbol{z}_2)$.

\begin{wrapfigure}{r}{0.48\columnwidth}
    \vspace{-2.5em} 
    \centering
\begin{tikzpicture}[scale=0.4]

    \def\geomLine{black!70}
    \def\geomText{black!65}
    \def\geomDash{black!40}

    \def\fmLine{magenta!65!black}
    \def\fmFill{magenta!60!black}
    \def\fmText{magenta!70!black}

    \def\dmLine{cyan!70!teal}
    \def\dmFill{cyan!65!teal}
    \def\dmText{cyan!60!black}

    \draw[very thick,\geomLine] (0,0) to[out=20,in=160] (8,1.5);
    \node[\geomText] at (4,2) {$\mathcal{M}$};
    
    \draw[oiVermillion,very thick,-{Stealth[length=6pt]}] (3.5,0.5) -- (6,1.2);
    \fill[oiVermillion] (3.5,0.5) circle (4pt) node[below,black!75] {$\bmz_1$};
    \fill[oiVermillion] (6,1.2) circle (4pt) node[right,black!75] {$\bmz_2$};
    \node[oiVermillion,align=left,font=\small] at (7.5,-0.5) {\textbf{Flow}\\\textbf{Matching}};
    
    \draw[oiBlue,very thick,-{Stealth[length=6pt]},smooth]
      plot[domain=0:1,samples=80]
      ({5 + (6-5)*\x + 1.0*(1-\x)*(1-\x)*sin(deg(10*pi*\x))},
       {8 + (1.2-8)*\x});
       
    \fill[oiBlue] (5,8) circle (4pt)
        node[below right,black!75] {$\bmz_0 \sim \mathcal{N}(0,I)$};
    \node[oiBlue,align=center,font=\small] at (2.2,7.0) {\textbf{Diffusion}};
    
    \draw[\geomDash,dashed] (5,8) -- (5,1.5);

\end{tikzpicture}\vspace{2pt}
\caption{
\textbf{Comparison between Flow Matching and Diffusion trajectories.}
Without loss of generality, we illustrate a two-view IMVC setting, where $\boldsymbol{z}_1$ denotes an observed view and $\boldsymbol{z}_2$ denotes a missing view. 
$\mathcal{M}$ denotes the region of the representation space that contains the cluster data, while $\boldsymbol{z}_0$ represents Gaussian noise. 
\textbf{\textcolor{oiBlue}{Diffusion trajectory}} deviates from $\mathcal{M}$ due to injected noise, resulting in a long and winding path when recovering $\boldsymbol{z}_2$. 
In contrast, the \textbf{\textcolor{oiVermillion}{Flow Matching trajectory}} remains within the data region $\mathcal{M}$, producing a more direct path that shortens as clustering improves.
}
\label{fig:fm_vs_diffusion_paths}
    
\vspace{-5em} 
\end{wrapfigure}
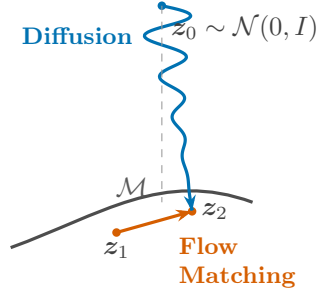

\subsubsection{Assumptions.}\label{sec:assumptions}
We establish all arguments under the following three assumptions.
Note that they are shared by both the conditional diffusion models and flow matching.

\begin{enumerate}[label=\textit{(\roman*)}]
    \item \label{assump:separable_representation}
    \textit{Cluster Separability.} The cluster sets $\mathcal{S}_k$ $(k=1,\ldots,K)$ in the representation space are compact subsets of $\mathbb{R}^d$. Any two clusters have a strictly positive minimum distance $\delta>0$, \ie $\inf_{i\neq j} d(\mathcal{S}_i,\mathcal{S}_j) = \delta > 0$.

    \item \label{assump:view_consistency}
    \textit{View Consistency.} Different-view representations of the same sample belong to the same cluster, \ie the clustering label mapping $c(\cdot)$ satisfies: for any sample representations $\boldsymbol{z}^{v}$ and $\boldsymbol{z}^{v'}$, $c(\boldsymbol{z}^{v}) = c(\boldsymbol{z}^{v'})$.

    \item \label{assump:perfect_approx}
    \textit{Sufficient Model Approximation.} 
    The neural network can sufficiently approximate the target function with bounded approximation error. \ie the model can capture the true probability flow or the vector field.
   
\end{enumerate}

\subsubsection{Preliminary on Flow Matching.}
\label{sec:preliminary-flow-matching}
In IMVC, to address the view-completion and the clustering objectives, flow matching directly constructs a deterministic ODE flow from the observed view representation $\bmz_1$ to the completed view representation $\bmz_2$.
Here, time domain $t \in [0,1]$, with $t=0$ for the observed view, $t=1$ for the completed view.
We note that flow matching does not involve random terms in its formulation, or noise initializations.

\paragraph{Interpolation Process.}

Flow matching first performs a linear interpolation between the observed view $\bmz_1$ (source distribution $p_0 = p(\bmz_1)$) and the completed view $\bmz_2$ (target distribution $p_1 = p(\bmz_2)$) to construct the interpolated representation $\bmx_t$ at time $t$.
The corresponding target vector field is the velocity of the interpolation path, \ie
\begin{align}
    \bmx_t = (1-t) \bmz_1 + t \bmz_2&, \quad t \in [0,1], \label{eqn:flow-1}\\
    \bmv_{\text{target}}(\bmx_t, t) = &\bmz_2 - \bmz_1, \label{eqn:flow-2}
\end{align}
where \eqnref{eqn:flow-1} represents the shortest path in Euclidean space, \ie a straight line path, which is the optimal shortcut connecting the observed view and the completed view.
Its path length is directly determined by $\|\bmz_2 - \bmz_1\|$.

\paragraph{Vector Field Learning.}
The learning objective of flow matching is to fit a vector field $\bmv_\theta(\bmx_t, t)$ via a neural network, such that it approximates the target vector field $\bmv_{\text{target}}(\bmx_t, t)$.
The training loss uses mean squared error:
\begin{equation}\label{eqn:mean-square-error}
    \mathcal{L}_{\text{FM}} = \mathbb{E}_{t \sim \mathcal{U}(0,1), \bmz_1, \bmz_2} \left\| \bmv_\theta(\bmx_t, t) - \bmv_{\text{target}}(\bmx_t, t) \right\|_2^2,
\end{equation}
where $\mathcal{U}(0,1)$ is the uniform distribution on $[0,1]$. The vector field $\bmv_\theta(\cdot)$ takes the interpolated representation $\bmx_t$, the time step $t$, and the observed view $\bmz_1$ as input and outputs a velocity vector of the same dimension as the representation.

\paragraph{Reversibility.}
The ODE framework of flow matching possesses a natural reversibility.
If the forward vector field $\bmv_\theta(\cdot)$ from $\bmz_1$ to $\bmz_2$ is learned, the reverse vector field $\bmv_{\theta_R}(\cdot)$ from $\bmz_2$ to $\bmz_1$ is simply $-\bmv_\theta(\cdot)$.
The reverse ODE is:
\begin{equation}
\label{eqn:reverse-ode}
    \frac{d\bmx_s}{ds} =\bmv_{\theta_R}(\bmx_{s},s)= -\bmv_\theta(\bmx_{1-t}, 1-t), \quad x_s = (1-s)\bmz_2+s\bmz_1.
\end{equation}

This property grants that flow matching only needs to train the vector field in one direction to achieve bidirectional view completion, where we provide a detailed proof in~\appref{app:integral_proof}.
This is in contrast to methods using diffusion models, which require training two separate models.

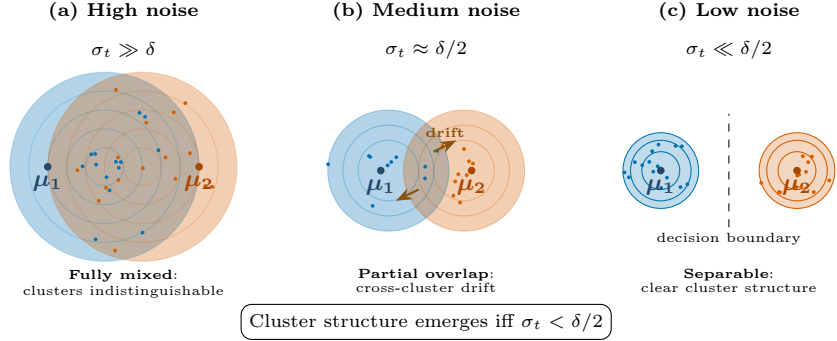
\begin{figure}[tp]
\centering
\begin{tikzpicture}[scale=0.5]

\begin{scope}[local bounding box=high]
\node at (0,3.2) {\scriptsize $\sigma_t \gg \delta$};
\begin{scope}[local bounding box=high+0.1,shift ={(0,0.1)}]

\fill[oiBlue,opacity=0.3] (-0.5,0) circle (2.5);
\fill[oiVermillion,opacity=0.3] (0.5,0) circle (2.5);

\foreach \r in {0.5,1,1.5,2,2.5} {
    \draw[oiBlue,thin,opacity=0.2] (-0.5,0) circle (\r);
    \draw[oiVermillion,thin,opacity=0.2] (0.5,0) circle (\r);
}

\pgfmathsetseed{42}
\foreach \i in {1,...,30} {
    \pgfmathrandominteger{\c}{1}{2}
    \ifnum\c=1
        \pgfmathrandominteger{\r}{-180}{180}
        \pgfmathrandominteger{\d}{0}{250}
        \fill[oiBlue] ({-0.5+\d/100*cos(\r)},{0+\d/100*sin(\r)}) circle (0.05);
    \else
        \pgfmathrandominteger{\r}{-180}{180}
        \pgfmathrandominteger{\d}{0}{250}
        \fill[oiVermillion] ({0.5+\d/100*cos(\r)},{0+\d/100*sin(\r)}) circle (0.05);
    \fi
}

\fill[mBlueEmph] (-2,0) circle (0.1) node[below] {$\boldsymbol{\mu_1}$};
\fill[oiVermillionEmph] (2,0) circle (0.1) node[below] {$\boldsymbol{\mu_2}$};
\end{scope}
\node[align=center, font=\tiny] at (0,-3)
{\textbf{Fully mixed}:\\ clusters indistinguishable};
\end{scope}

\begin{scope}[local bounding box=med,shift={(8,0)}]
\node at (0,3.2) {\scriptsize $\sigma_t \approx \delta/2$};

\fill[oiBlue,opacity=0.3] (-1,0) circle (1.6);
\fill[oiVermillion,opacity=0.3] (1,0) circle (1.6);

\foreach \r in {0.8,1.2,1.6} {
    \draw[oiBlue,thin,opacity=0.3] (-1,0) circle (\r);
    \draw[oiVermillion,thin,opacity=0.3] (1,0) circle (\r);
}

\pgfmathsetseed{43}
\foreach \i in {1,...,20} {
    \pgfmathrandominteger{\c}{1}{2}
    \ifnum\c=1
        \pgfmathrandominteger{\r}{-180}{180}
        \pgfmathrandominteger{\d}{0}{160}
        \fill[oiBlue] ({-1+\d/100*cos(\r)},{0+\d/100*sin(\r)}) circle (0.05);
    \else
        \pgfmathrandominteger{\r}{-180}{180}
        \pgfmathrandominteger{\d}{0}{160}
        \fill[oiVermillion] ({1+\d/100*cos(\r)},{0+\d/100*sin(\r)}) circle (0.05);
    \fi
}

\draw[-{Stealth[length=6pt]},thick,oiBrown] (0.2,0.5) -- (0.8,0.8) node[midway,above] {\tiny \textbf{drift}};
\draw[-{Stealth[length=6pt]},thick,oiBrown] (-0.2,-0.5) -- (-0.8,-0.8);

\fill[mBlueEmph] (-1.2,0) circle (0.1) node[below] {$\boldsymbol{\mu_1}$};
\fill[oiVermillionEmph] (1.2,0) circle (0.1) node[below] {$\boldsymbol{\mu_2}$};

\node[align=center, font=\tiny] at (0,-2.9)
{\textbf{Partial overlap}:\\[-1pt] cross-cluster drift};
\end{scope}

\begin{scope}[local bounding box=low,shift={(16,0)}]
\node at (0,3.2) {\scriptsize $\sigma_t \ll \delta/2$};

\fill[oiBlue,opacity=0.25] (-1.8,0) circle (1.0);
\fill[oiVermillion,opacity=0.25] (1.8,0) circle (1.0);

\foreach \r in {0.5,0.8,1.0} {
    \draw[oiBlue,thin] (-1.8,0) circle (\r);
    \draw[oiVermillion,thin] (1.8,0) circle (\r);
}

\draw[dashed] (0,-1.5) -- (0,1.5);
\node at (0,-1.8) {\tiny decision boundary};

\pgfmathsetseed{44}
\foreach \i in {1,...,15} {
    \pgfmathrandominteger{\r}{-180}{180}
    \pgfmathrandominteger{\d}{0}{100}
    \fill[oiBlue] ({-1.8+\d/100*cos(\r)},{0+\d/100*sin(\r)}) circle (0.05);
    
    \pgfmathrandominteger{\r}{-180}{180}
    \pgfmathrandominteger{\d}{0}{100}
    \fill[oiVermillion] ({1.8+\d/100*cos(\r)},{0+\d/100*sin(\r)}) circle (0.05);
}

\fill[mBlueEmph] (-1.8,0) circle (0.1) node[below] {$\boldsymbol{\mu_1}$};
\fill[oiVermillionEmph] (1.8,0) circle (0.1) node[below] {$\boldsymbol{\mu_2}$};

\node[align=center, font=\tiny] at (0,-2.9)
{\textbf{Separable}:\\[-1pt] clear cluster structure};
\end{scope}

\node[font=\bfseries] at (0,4.2) {\scriptsize (a) High noise};
\node[font=\bfseries] at (8,4.2) {\scriptsize (b) Medium noise};
\node[font=\bfseries] at (16,4.2) {\scriptsize (c) Low noise};

\node[
    draw,
    rounded corners,
    fill=white,
    font=\scriptsize
] at (8,-4) 
{
    Cluster structure emerges iff $\sigma_t < \delta/2$
};
\end{tikzpicture}
\vspace{-7pt}

\caption{
\textbf{Cluster separability under diffusion noise.}
\textbf{(a)} High noise ($\sigma_t \gg \delta$): supports fully overlap, samples mix heavily;
\textbf{(b)} Medium noise ($\sigma_t \approx \delta/2$): partial overlap may cause cross-cluster drift;
\textbf{(c)} Low noise ($\sigma_t \ll \delta/2$): clusters are well separated with a clear boundary.
In finite-step inference, $\sigma_t$ is often large due to limited steps, making scenario (b) common and leading to cross-cluster drift.
We refer readers to \appref{sec:conditional-diffusion-process}, \eqnref{eqn:property3_pre1} for detailed analysis.
}
\label{fig:diffusion-seperate}
\end{figure}

\subsubsection{Flow Matching in IMVC.}\label{sec:benefits-of-flow-matching}

With the preliminaries of flow matching introduced above, we argue that it is structurally more aligned to the task of IMVC, compared to existing diffusion-based methods.
We concretely demonstrate our claim through the two corollaries below, while relating each one to the limitations of the diffusion models.
A complete theoretical analysis regarding the properties and limitations of the diffusion model is provided in~\appref{sec:mathematical-diffusion}, detailed proofs of the corollaries are in~\appref{sec:prove-corollary}.

\begin{corollary}[Flow Matching's Shortcut Path]\label{corollary_2}
Flow matching constructs an optimal shortcut path between views representation.

\end{corollary}

As illustrated in~\figref{fig:fm_vs_diffusion_paths}, during the training process of IMVC, view representations gradually align as $\|\bmz_1 - \bmz_2\|$ decreases. 
The path length in flow matching shortens adaptively. When views are perfectly aligned, the path length becomes a small value.
On the other hand, the diffusion model's starting point is global Gaussian noise unrelated to the cluster class-conditional data distributions, which encode the clustering structure (\appref{sec:conditional-diffusion-process}, \eqnref{eqn:property2}). The reverse process must complete the full evolution of ``noise space $\Rightarrow$ representation space $\Rightarrow$ cluster-specific distribution''.
Even if the observed view representation and completion target view representation in latent space are highly aligned, diffusion cannot utilize the view consistency information.
The path length is fixed as the distance from the noise to the cluster set. 

\begin{remark}[Shortcut Path Advantage]
    Compared to the diffusion model's Gaussian noise initialization (\eqnref{eqn:property2}), flow matching offers the theoretical advantages of non-redundant paths and adaptively shortening path length, as shown in~\figref{fig:fm_vs_diffusion_paths}.
\end{remark}

\begin{corollary}[Flow Matching's Finite-Step Separability]\label{corollary_3}
Flow matching achieves cluster separability within any finite number of steps.
\end{corollary}

In other words, sample trajectories from different cluster sets remain separated throughout the evolution, with no cross-cluster drift, as in~\figref{fig:diffusion-seperate}~\textbf{(c)}.
In contrast, the cluster separability of the diffusion model relies on a strong asymptotic condition (where $\sigma_t < \delta/2$), which only holds as $t \to 0$, as shown in~\figref{fig:diffusion-seperate} and is discussed in detail in~\appref{sec:conditional-diffusion-process}, \eqnref{eqn:property3_pre1}.
This condition is difficult to satisfy in finite-step practical inference, leading to overlapping supports of noised distributions for different clusters and making samples prone to cross-cluster drift, significantly reducing cluster separability.

\begin{remark}[Finite-Step Separability Advantage]
    Flow matching offers the theoretical benefit of finite-step cluster separability with no cross-cluster drift compared to the diffusion model's asymptotic separability condition (\figref{fig:diffusion-seperate}, \eqnref{eqn:property3_pre1}).
\end{remark}

\subsection{Flow-Matching Formulation of IMVC}\label{sec:flow-based-method}

Motivated by the theoretical analysis in~\secref{sec:theoretical-analysis}, we formulate incomplete multi-view clustering as a cross-view completion problem in the latent representation space. Under the assumptions of cluster separability, view consistency, and sufficient model approximation, our method is designed to preserve clustering structure while recovering missing-view representations. The overall framework consists of three components: within-view reconstruction, cross-view alignment (including cluster-level and entropic alignment), and straight-path completion.

\begin{figure}[tp]
    \centering
    \includegraphics[width=1\linewidth]{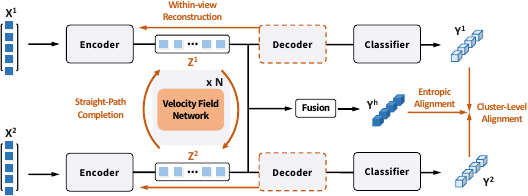}
\caption{
\textbf{Overview of the proposed flow-matching framework for incomplete multi-view clustering.}
Each view is encoded into latent representations and reconstructed via decoders.
A velocity field network models straight-path transport between paired latent for cross-view completion.
The completed representations are fused for cluster prediction,
followed by cluster-level alignment and entropic alignment.
}
    \label{fig:main_fig}
\end{figure}

\subsubsection{Within-view Reconstruction.}
Before performing cross-view transport, we first learn compact latent representations for each view via view-specific autoencoders.
Given an input feature $\bmx_v \in X$ from view $v \in \{1,2\}$, an encoder $\mathcal{E}_v(\cdot)$ maps it to a latent representation $\bmz_v = \mathcal{E}_v(\bmx_v) \in Z$, and a decoder $\mathcal{D}_v(\cdot)$ reconstructs it as $\hat{\bmx}_v = \mathcal{D}_v(\bmz_v)$.
This step extracts view-specific structural information and learns latent representations for each view.
To preserve the intrinsic structure of each view in the latent space, we adopt a standard reconstruction objective in $N$ samples:
\begin{equation}
\mathcal{L}_{\text{rec}} = \frac{1}{N} \sum_{i=1}^N \left( \| \mathcal{D}_1(\mathcal{E}_1(\boldsymbol{x}_1^{(i)})) - \boldsymbol{x}_1^{(i)} \|^2 + \| \mathcal{D}_2(\mathcal{E}_2(\boldsymbol{x}_2^{(i)})) - \boldsymbol{x}_2^{(i)} \|^2 \right),
\end{equation}

\subsubsection{Cluster-Level Alignment.}

Under Assumptions~\ref{assump:separable_representation}–\ref{assump:view_consistency}, different views of the same sample are expected to belong to the same cluster and therefore share a consistent cluster assignment. However, in practice, the view-specific encoders and the learned clustering head can yield view-dependent predictions, leading to inconsistent assignments across views and degrading the clustering structure of the learned representations. We therefore employ a cluster-level contrastive objective to explicitly couple the two views in the clustering space.

Specifically, let $\bmy_1,\bmy_2\in\mathbb{R}^{K}$ be the cluster probability vectors 
predicted from $\bmz_1$ and $\bmz_2$, respectively.
We adopt an InfoNCE \cite{oord2018representation} objective to align cross-view cluster predictions in $N$ samples:
\begin{equation}
\mathcal{L}_{\text{cluster}}
=
-\frac{1}{N}
\sum_{i=1}^{N}
\log
\frac{\exp\!\left(\cos(\bmy_1^{(i)}, \bmy_2^{(i)})/\tau\right)}
{\sum_{j=1}^{N}
\exp\!\left(\cos(\bmy_1^{(i)}, \bmy_2^{(j)})/\tau\right)},
\end{equation}
where $\cos(\cdot,\cdot)$ denotes cosine similarity, 
and $\tau$ is a temperature parameter.
This objective increases the similarity between corresponding cross-view predictions 
while discouraging alignment with non-matching samples, 
thereby promoting view-consistent clustering and preserving inter-cluster separation.

\subsubsection{Entropic Alignment.}\label{sec:entropic-alignment}
Although the cluster-level contrastive objective promotes cross-view assignment consistency, it operates at the sample level and does not explicitly regularize the structure of the predicted probability distributions.
In practice, direct fusion in Euclidean space may distort the relative probability structure and over-smooth cluster confidence.
We therefore employ an entropy-weighted alignment objective 
to regularize probability consistency across views.

Let $\bmy_1$, $\bmy_2$, and $\bmy_h$ denote the cluster probability distributions 
predicted from $\bmz_1$, $\bmz_2$, and the fused representation $\bmh$, respectively.
Since these distributions lie on the probability simplex, 
we adopt the centered log-ratio (clr) transformation (\appref{sec:clr-proof})
to perform fusion in a simplex-consistent space.
An entropy-based weighting scheme assigns larger weights to more confident predictions, 
and the target distribution $\bmy^*$ is obtained via a weighted combination in the clr space 
followed by inverse mapping back to the simplex.
We then enforce consistency between the fused prediction and the entropy-weighted target via KL divergence:
\begin{equation}
\mathcal{L}_{\text{entropy}}
=
\mathrm{KL}(\bmy_h \,\|\, \bmy^*),
\end{equation}
where $\KL{\cdot}{\cdot}$ is the KL divergence. This objective stabilizes cross-view probability structure 
while preventing uncertain predictions from dominating the fusion, 
thereby preserving post-alignment cluster separability.

\subsubsection{Straight-Path Completion.}

With cluster-level and entropic alignment regularizing cross-view consistency in the clustering space,
we now realize cross-view representation completion via flow matching in the latent space.
Given paired latent representations $(\bmz_1,\bmz_2)$ from complete samples,
we construct the interpolated state $\bmx_t$ along the straight-line path (\eqnref{eqn:flow-1}).
A neural vector field $\mathcal{F}(\cdot)$ is trained to approximate the target transport direction,
so that integrating the learned ODE maps an observed representation to its missing-view counterpart.

\paragraph{Flow Prediction Loss.}

Given the straight-line interpolation $\bmx_t$ defined in~\eqnref{eqn:flow-1}, the target transport direction at any intermediate state is uniquely determined as $\bmz_2-\bmz_1$ (see~\eqnref{eqn:flow-2}). 
We therefore learn a parametric vector field $\mathcal{F}(\bmx_t,t)$ to approximate this target velocity by minimizing
\begin{equation}
\mathcal{L}_{\text{flow-pred}}
=
\mathbb{E}_{t,\bmz_1,\bmz_2}
\left\|
\mathcal{F}(\bmx_t,t)
-
(\bmz_2-\bmz_1)
\right\|_2^2 .
\end{equation}

This objective enforces the learned vector field to follow the straight-line shortcut path
between paired latent representations,
as discussed in Corollary~\ref{corollary_2}.
By aligning the local velocity with $\bmz_2-\bmz_1$, the transport process exploits view consistency and adaptively shortens the path length as cross-view representations become aligned.
This avoids redundant detours and maintains the cluster structure established by the alignment modules.

\paragraph{Completion Consistency Loss.}

Although the vector field is trained to approximate the shortcut transport direction,
finite-step numerical integration may introduce approximation errors.
Based on the reversibility property of flow matching (Eq.~\ref{eqn:reverse-ode}),
forward integration ($t:0\rightarrow1$) maps $\bmz_1$ to the recovered feature $\hat{\bmz}_2$,
while reverse integration ($s:0\rightarrow1, s = 1-t$) maps $\bmz_2$ to $\hat{\bmz}_1$:
\begin{align}
\hat{\bmz}_2 &= \bmz_1 + \int_{0}^{1} 
\mathcal{F}\!\left(\bmx_t, t\right)\, \mathrm{d}t, \\
\hat{\bmz}_1 &= \bmz_2 + \int_{0}^{1} 
-\mathcal{F}\!\left(\bmx_{1-t}, 1-t\right)\, \mathrm{d}s.
\end{align}
We then enforce bidirectional completion consistency by minimizing
\begin{equation}
\mathcal{L}_{\text{flow-comp}}
=
\mathbb{E}_{\bmz_1,\bmz_2}
\left[
\|\hat{\bmz}_2 - \bmz_2\|_2^2
+
\|\hat{\bmz}_1 - \bmz_1\|_2^2
\right].
\end{equation}

This loss constrains the transport trajectory at both endpoints, ensuring that finite-step evolution remains faithful to the true latent representations.
Together with the shortcut-path property, it supports stable cross-view completion while maintaining the cluster structure discussed in Corollary~\ref{corollary_3}.

\subsubsection{Overall Objective.}

The model is trained by minimizing a weighted combination of the losses introduced above,
which jointly optimize representation reconstruction,
cross-view cluster consistency,
probability-level alignment,
and shortcut-path completion:
\begin{equation}
\mathcal{L}
=
\mathcal{L}_{\text{rec}}
+
\lambda_1 \mathcal{L}_{\text{flow-pred}}
+
\lambda_2 \mathcal{L}_{\text{flow-comp}}
+
\lambda_3 \mathcal{L}_{\text{cluster}}
+
\lambda_4 \mathcal{L}_{\text{entropy}},
\end{equation}
where $\lambda_1,\lambda_2,\lambda_3,$ and $\lambda_4$ balance the contributions of the transport, alignment, and reconstruction objectives.

\section{Experiment}
\label{sec:experiment}

This section presents the experimental evaluation of the proposed framework. 
We first describe the experimental settings (\secref{sec:settings}),
and compare our method with recent incomplete multi-view clustering approaches on multiple benchmarks (\secref{sec:comparison}). 
Then, we conduct ablation studies to analyze the contribution of each component (\secref{sec:ablation}). 
Finally, we provide additional analyses, including hyperparameter sensitivity, representation visualization, and flow-step behavior, to further understand the properties of the proposed method (\secref{sec:analysis}).

\subsection{Experiment Setup}
\label{sec:settings}

\mypara{Dataset.}
We evaluate the proposed method on five multi-view benchmarks: 
\textit{(i)} Synthetic3D~\cite{kumar2011co}, a synthetic Gaussian-mixture dataset with correlated features, where two views are selected following~\cite{zhang2025incomplete}; 
\textit{(ii)}  CUB~\cite{wah2011caltech}, a fine-grained bird dataset containing 11,788 image–text pairs from 200 categories, where visual features are extracted using GoogLeNet and textual features via doc2vec, and the first 10 categories are used as in~\cite{zhang2025incomplete,li2025community}; 
\textit{(iii)}  HandWritten~\cite{asuncion2007uci}, a digit dataset with 2,000 samples and six feature views, from which two views are adopted following~\cite{zhang2025incomplete}; 
\textit{(iv)}  LandUse-21~\cite{yang2010bag}, a remote-sensing dataset with 2,100 images from 21 categories, where PHOG and LBP descriptors are used as two views; 
and \textit{(v)}  Fashion~\cite{xiao2017fashion}, a product image dataset with 10,000 samples, from which two visual feature views are selected.

\mypara{Metric.}
Following the standard incomplete multi-view clustering (IMVC) protocol~\cite{lin2021completer,zhang2025incomplete}, we simulate missing views by randomly masking one view for a subset of samples. The missing rate is defined as the ratio of incomplete samples to the total dataset size.
We follow~\cite{lin2021completer,zhang2025incomplete} to evaluate method performance using three standard metrics: Accuracy (ACC), Normalized Mutual Information (NMI), and Adjusted Rand Index (ARI).

\mypara{Implementation details.}
Our framework consists of view-specific autoencoders, a straight-path flow matching module, and a clustering classifier.
The model is trained for $200$ epochs using the Adam optimizer with an initial learning rate of $1 \times 10^{-3}$. 
The batch size is set according to the dataset configuration (typically $128$ or $256$).
We simulate incomplete multi-view scenarios by randomly masking one view for a subset of samples based on the target missing rate ($\rho \in \{0.1, 0.3, 0.5\}$). 
The weighting coefficients $\lambda_1, \lambda_2, \lambda_3,$ and $\lambda_4$ are configured individually for each dataset.
For the flow matching process, we adopt an efficient setting with a single integration step ($Step=1$) during both training and inference.
We refer readers to \appref{sec:experimental-supplement} for detailed implementation details.

\begin{table}[htbp]
\centering
\caption{
\textbf{Comparison with state-of-the-art incomplete multi-view clustering (IMVC) methods under different missing rates.}
\textbf{ACC}, \textbf{NMI}, and \textbf{ARI} denote clustering Accuracy, Normalized Mutual Information, and Adjusted Rand Index, respectively (higher is better).
$\tau$ indicates the missing rate of views.
}
\label{tab:performance}
\begin{adjustbox}{scale=0.8}
\begin{tabular*}{1.23\textwidth}{@{\extracolsep{\fill}} llccccccccc}
\toprule
\multirow{2}{*}{\textbf{Dataset}} & \multirow{2}{*}{\textbf{Method}} & \multicolumn{3}{c}{\textbf{$\tau=0.1$}} & \multicolumn{3}{c}{\textbf{$\tau=0.3$}} & \multicolumn{3}{c}{\textbf{$\tau=0.5$}} \\
\cmidrule(lr){3-5} \cmidrule(lr){6-8} \cmidrule(lr){9-11}
& & \textbf{ACC} & \textbf{NMI} & \textbf{ARI} & \textbf{ACC} & \textbf{NMI} & \textbf{ARI} & \textbf{ACC} & \textbf{NMI} & \textbf{ARI} \\
\midrule
\multirow{10}{*}{\textbf{Synthetic3d}} 
& DCP\cite{lin2022dual} & 88.00 & 65.17 & 67.87 & 79.83 & 55.21 & 56.69 & 85.53 & 58.33 & 62.56 \\
& DSIMVC\cite{tang2022deep} & 73.11 & 59.02 & 58.38 & 70.78 & 56.41 & 55.43 & 67.44 & 51.74 & 49.42 \\
& GCFAGG\cite{yan2023gcfagg}  & 72.52 & 57.87 & 56.68 & 70.23 & 55.51 & 53.34 & 69.12 & 53.21 & 52.03 \\
& CPSPAN\cite{jin2023deep}  & 88.83 & 65.51 & 69.12 & 87.50 & 61.79 & 66.68 & 86.33 & 59.29 & 63.65 \\
& APADC\cite{xu2023adaptive} & 85.73 & 59.07 & 62.98 & 86.47 & 60.84 & 64.90 & 84.38 & 58.48 & 61.85 \\
& ProImp\cite{li2023incomplete}  & 86.55 & 61.79 & 65.42 & 85.50 & 59.29 & 62.81 & 82.78 & 54.76 & 57.69 \\
& DVIMVC\cite{xu2024deep} & 50.03 & 28.43 & 24.57 & 46.97 & 22.07 & 15.93 & 50.32 & 25.64 & 18.79 \\
& ICMVC\cite{chao2024incomplete} & 85.03 & 58.30 & 61.96 & 87.17 & 61.20 & 66.29 & 84.20 & 55.33 & 59.60 \\
& DCG\cite{zhang2025incomplete} & 91.23 & 71.35 & 76.05 & 88.00 & 64.11 & 68.48 & 87.67 & 63.23 & 67.52 \\
& \textbf{Ours} & \textbf{94.17} & \textbf{78.66} & \textbf{83.51} & \textbf{91.17} & \textbf{69.56} & \textbf{75.56} & \textbf{88.17} & \textbf{63.78} & \textbf{68.26} \\
\midrule
\multirow{10}{*}{\textbf{CUB}} 
& DCP\cite{lin2022dual} & 42.77 & 55.42 & 33.39 & 40.60 & 52.37 & 31.41 & 38.87 & 50.15 & 31.18 \\
& DSIMVC\cite{tang2022deep} & 63.67 & 59.85 & 46.23 & 49.22 & 49.93 & 32.80 & 48.99 & 49.61 & 32.35 \\
& GCFAGG\cite{yan2023gcfagg} & 67.67 & 64.14 & 51.14 & 62.72 & 60.27 & 45.03 & 59.63 & 55.95 & 39.91 \\
& CPSPAN\cite{jin2023deep} & 76.67 & 71.38 & 58.65 & 74.33 & 70.33 & 56.84 & 73.33 & 69.68 & 57.17 \\
& APADC\cite{xu2023adaptive} & 53.04 & 59.05 & 42.98 & 70.33 & 70.33 & 56.84 & 48.53 & 59.93 & 41.54 \\
& ProImp\cite{li2023incomplete} & 69.33 & 69.65 & 55.93 & 71.83 & 64.45 & 51.64 & 69.56 & 64.19 & 52.06 \\
& DVIMVC\cite{xu2024deep} & 63.83 & 63.00 & 50.50 & 70.42 & 68.43 & 56.35 & 68.98 & 68.81 & 52.61 \\
& ICMVC\cite{chao2024incomplete} & 32.33 & 33.36 & 16.85 & 45.13 & 41.32 & 25.11 & 41.70 & 37.78 & 21.62 \\
& DCG\cite{zhang2025incomplete} & 82.23 & 77.70 & 69.21 & 77.17 & 71.35 & 59.85 & 75.50 & \textbf{72.21} & 59.12 \\
& \textbf{Ours} & \textbf{86.00} & \textbf{78.19} & \textbf{72.28} & \textbf{82.67} & \textbf{72.80} & \textbf{66.31} & \textbf{80.50} & 71.74 & \textbf{63.15} \\
\midrule
\multirow{10}{*}{\textbf{HandWritten}} 
& DCP\cite{lin2022dual} & 63.66 & 70.44 & 52.82 & 75.68 & 79.05 & 67.45 & 72.07 & 76.17 & 63.81 \\
& DSIMVC\cite{tang2022deep} & 71.37 & 70.24 & 60.30 & 67.47 & 65.27 & 54.26 & 56.90 & 57.99 & 44.12 \\
& GCFAGG\cite{yan2023gcfagg} & 75.07 & 68.19 & 59.79 & 69.16 & 63.25 & 53.67 & 61.02 & 53.57 & 42.12 \\
& CPSPAN\cite{jin2023deep} & 68.70 & 69.06 & 59.17 & 82.20 & 75.05 & 69.58 & 72.55 & 68.60 & 57.51 \\
& APADC\cite{xu2023adaptive} & 72.58 & 71.29 & 59.07 & 60.66 & 64.84 & 48.75 & 52.70 & 63.12 & 39.96 \\
& ProImp\cite{li2023incomplete} & 81.05 & 78.48 & 70.40 & 80.60 & 77.26 & 69.72 & 78.98 & 74.08 & 66.47 \\
& DVIMVC\cite{xu2024deep} & 29.08 & 19.36 & 11.06 & 26.63 & 16.57 & 8.84 & 25.13 & 15.69 & 8.09 \\
& ICMVC\cite{chao2024incomplete} & 82.70 & 81.06 & 74.67 & 82.27 & 79.90 & 72.81 & 75.04 & 71.89 & 63.29 \\
& DCG\cite{zhang2025incomplete} & 82.75 & 82.63 & 74.88 & 82.70 & 80.54 & 73.96 & 80.80 & 76.21 & 70.45 \\
& \textbf{Ours} & \textbf{85.30} & \textbf{82.78} & \textbf{76.72} & \textbf{86.50} & \textbf{82.09} & \textbf{77.04} & \textbf{83.40} & \textbf{76.26} & \textbf{70.76} \\
\midrule
\multirow{10}{*}{\textbf{LandUse-21}} 
& DCP\cite{lin2022dual} & 26.19 & 31.20 & 13.02 & 25.48 & 30.18 & 11.08 & 21.52 & 26.32 & 11.11 \\
& DSIMVC\cite{tang2022deep} & 16.56 & 16.40 & 4.32 & 16.46 & 16.57 & 4.35 & 16.70 & 16.84 & 4.50 \\
& GCFAGG\cite{yan2023gcfagg} & 19.05 & 19.99 & 6.32 & 18.53 & 19.89 & 6.24 & 18.43 & 19.47 & 5.98 \\
& CPSPAN\cite{jin2023deep} & 20.05 & 27.20 & 8.14 & 18.05 & 25.75 & 6.90 & 20.38 & 26.99 & 8.05 \\
& APADC\cite{xu2023adaptive} & 20.92 & 26.74 & 7.97 & 19.99 & 24.40 & 7.44 & 19.15 & 23.91 & 7.08 \\
& ProImp\cite{li2023incomplete} & 24.43 & 29.22 & 11.44 & 24.63 & 28.45 & 11.78 & 24.45 & 27.43 & 10.94 \\
& DVIMVC\cite{xu2024deep} & 13.90 & 22.48 & 2.66 & 13.24 & 23.21 & 2.74 & 12.86 & 19.67 & 2.54 \\
& ICMVC\cite{chao2024incomplete} & 26.81 & 30.71 & 13.60 & 26.12 & 29.83 & 12.12 & 26.12 & 29.83 & 12.12 \\
& DCG\cite{zhang2025incomplete} & 27.52 & 31.36 & 14.57 & 27.33 &  \textbf{32.09} & 14.47 & 25.76 & 29.07 & 13.17 \\
& \textbf{Ours} & \textbf{29.10} & \textbf{31.95} & \textbf{15.15} & \textbf{28.19} & 31.71 & \textbf{14.55} & \textbf{27.24} & \textbf{30.84} & \textbf{13.66} \\
\midrule
\multirow{10}{*}{\textbf{Fashion}} 
& DCP\cite{lin2022dual} & 83.70 & 84.30 & 76.50 & 71.80 & 70.90 & 52.50 & 60.80 & 59.50 & 33.10 \\
& DSIMVC\cite{tang2022deep} & 88.00 & 86.40 & 81.10 & 87.30 & 85.00 & 78.90 & 83.50 & 80.30 & 73.70 \\
& GCFAGG\cite{yan2023gcfagg} & 78.21 & 74.50 & 66.28 & 76.34 & 72.53 & 63.98 & 74.47 & 69.83 & 60.37 \\
& CPSPAN\cite{jin2023deep} & 66.16 & 68.45 & 55.73 & 64.80 & 68.22 & 55.55 & 54.81 & 64.13 & 48.84 \\
& APADC\cite{xu2023adaptive} & 81.40 & 86.50 & 73.30 & 80.90 & 85.01 & 73.10 & 75.40 & 81.50 & 67.60 \\
& ProImp\cite{li2023incomplete} & 92.88 & 88.34 & 86.09 & 74.11 & 76.97 & 66.42 & 89.76 & 81.94 & 79.93 \\
& DVIMVC\cite{xu2024deep} & 79.38 & 80.22 & 71.50 & 82.26 & 79.82 & 72.46 & 80.17 & 77.09 & 69.51 \\
& ICMVC\cite{chao2024incomplete} & 92.41 & 87.05 & 85.11 & 89.31 & 83.06 & 79.93 & 79.37 & 74.44 & 68.46 \\
& DCG\cite{zhang2025incomplete} & 95.83 & 91.29 & 91.19 & 93.13 & 86.99 & 86.00 & 90.04 & 82.25 & 79.99 \\
& \textbf{Ours} & \textbf{96.07} & \textbf{91.99} & \textbf{91.78} & \textbf{93.66} & \textbf{87.93} & \textbf{87.07} & \textbf{90.09} & \textbf{82.83} & \textbf{80.58} \\
\bottomrule
\end{tabular*}
\end{adjustbox}
\end{table}

\subsection{Comparisons with State-of-the-Arts.}
\label{sec:comparison}

Table~\ref{tab:performance} compares our method with recent incomplete multi-view clustering approaches under three missing rates on five benchmark datasets. 
Overall, the proposed framework consistently achieves the best or competitive performance across almost all evaluation metrics.
On datasets with clear cluster structures, such as Synthetic3D and HandWritten, our method achieves notable improvements across all metrics. 
These results suggest that the straight-path completion learned by flow matching can recover missing representations while maintaining the underlying cluster geometry.

The advantage of our method becomes more pronounced under higher missing rates. 
As $\tau$ increases from 0.1 to 0.5, many existing methods experience notable performance degradation due to reduced cross-view correspondence, whereas our approach remains relatively stable across datasets. 
This robustness arises from the combination of alignment and flow-based completion: the cluster-level and entropic alignment modules enforce view-consistent assignments, while the flow-matching transport recovers missing representations by following the straight-path transport between paired latent representations. 
This behavior is consistent with our theoretical analysis that flow matching maintains cluster separability under finite-step evolution, resulting in stronger robustness under severe incompleteness.

\begin{table}[t]
\centering
\caption{\textbf{Ablation study of the proposed loss components on the HandWritten dataset under missing rate $\tau=0.3$.}
$\mathcal{L}_{\text{flow-pred}}$ and $\mathcal{L}_{\text{flow-comp}}$ denote the flow-based shortcut completion losses, 
$\mathcal{L}_{\text{cluster}}$ represents the cluster-level alignment loss, 
and $\mathcal{L}_{\text{entropy}}$ refers to the entropy-based alignment loss.
\cmark\ and \xmark\ indicate whether a component is enabled or removed. All variants retain the reconstruction loss.}
\label{tab:handwrite_exp}

\begin{adjustbox}{scale=0.85}
\begin{tabular*}{1.15\textwidth}{@{\extracolsep{\fill}} p{0.34\textwidth} cccc ccc}
\toprule
\multirow{2}{*}{\textbf{Variants}} 
& \multicolumn{4}{c}{\textbf{Loss Components}} 
& \multicolumn{3}{c}{\textbf{Metrics (\%)}} \\
\cmidrule(lr){2-5} \cmidrule(lr){6-8}
& $\mathcal{L}_{\text{flow-pred}}$ 
& $\mathcal{L}_{\text{flow-comp}}$ 
& $\mathcal{L}_{\text{cluster}}$ 
& $\mathcal{L}_{\text{entropy}}$ 
& \textbf{ACC} & \textbf{NMI} & \textbf{ARI} \\
\midrule
(w/o) $\mathcal{L}_{\text{flow-pred}}$ & \xmark & \cmark & \cmark & \cmark & 73.95 & 70.92 & 61.42 \\
(w/o) $\mathcal{L}_{\text{flow-comp}}$ & \cmark & \xmark & \cmark & \cmark & 76.15 & 71.86 & 63.38 \\
(w/o) $\mathcal{L}_{\text{flow-pred}}$ \& $\mathcal{L}_{\text{flow-comp}}$ & \xmark & \xmark & \cmark & \cmark & 77.35 & 71.89 & 64.16 \\
(w/o) $\mathcal{L}_{\text{cluster}}$ & \cmark & \cmark & \xmark & \cmark & 27.75 & 37.13 & 20.05 \\
(w/o) $\mathcal{L}_{\text{entropy}}$ & \cmark & \cmark & \cmark & \xmark & 77.20 & 71.79 & 66.69 \\
(w/o) $\mathcal{L}_{\text{cluster}}$ \& $\mathcal{L}_{\text{entropy}}$ & \cmark & \cmark & \xmark & \xmark & 28.70 & 17.67 & 7.75 \\
\midrule
\textbf{Full model} & \cmark & \cmark & \cmark & \cmark & \textbf{86.50} & \textbf{82.09} & \textbf{77.04} \\
\bottomrule
\end{tabular*}
\end{adjustbox}
\end{table}

\subsection{Ablation Study}
\label{sec:ablation}

We conduct ablation experiments on HandWritten with a missing rate $\tau=0.3$ to analyze the contribution of each component in the proposed flow-matching formulation, as summarized in Table~\ref{tab:handwrite_exp}.

\subsubsection{Effect of $\mathcal{L}_{\text{cluster}}$.}
Removing the cluster-level contrastive loss causes the most severe performance drop (ACC decreases to 27.75\%). 
This loss explicitly aligns cross-view cluster assignments and establishes view-consistent clustering structure in the latent space. 
Without this constraint, representations from different views may produce inconsistent cluster predictions, violating the cluster separability assumption required by the flow-matching formulation. 
Consequently, reliable cross-view completion cannot be achieved.

\mypara{Effect of $\mathcal{L}_{\text{entropy}}$.}
The entropy-based alignment complements the cluster-level contrastive objective by further stabilizing cross-view probability predictions.
While $\mathcal{L}_{\text{cluster}}$ enforces assignment consistency at the sample level, $\mathcal{L}_{\text{entropy}}$ regularizes the probability distributions across views. 
Removing this term leads to performance degradation, indicating that probability-level alignment is also important for maintaining stable cluster assignments and cross-view completion.

\mypara{Effect of $\mathcal{L}_{\text{flow-pred}}$.}
The flow prediction loss supervises the learned vector field to follow the straight-path transport direction between paired latent representations. 
Removing this loss weakens the ability of the model to learn a consistent transport direction, resulting in degraded completion quality. 
This confirms that accurate velocity-field estimation is necessary for modeling cross-view transport.

\mypara{Effect of $\mathcal{L}_{\text{flow-comp}}$.}
The completion consistency loss enforces bidirectional completion consistency between paired latent representations. 
Without this constraint, the recovered features may deviate from the true latent representations during finite-step integration, leading to degraded performance. 
This loss therefore improves the stability of cross-view completion.
When both flow losses are removed, the model reduces to an alignment-only variant and achieves moderate performance. 
Introducing the flow-matching formulation substantially improves all metrics, demonstrating that straight-path transport effectively exploits the clustering structure established by the alignment modules.

\begin{figure}[t]
    \centering
    \begin{subfigure}[b]{0.25\textwidth}
        \centering
        \includegraphics[width=\textwidth]{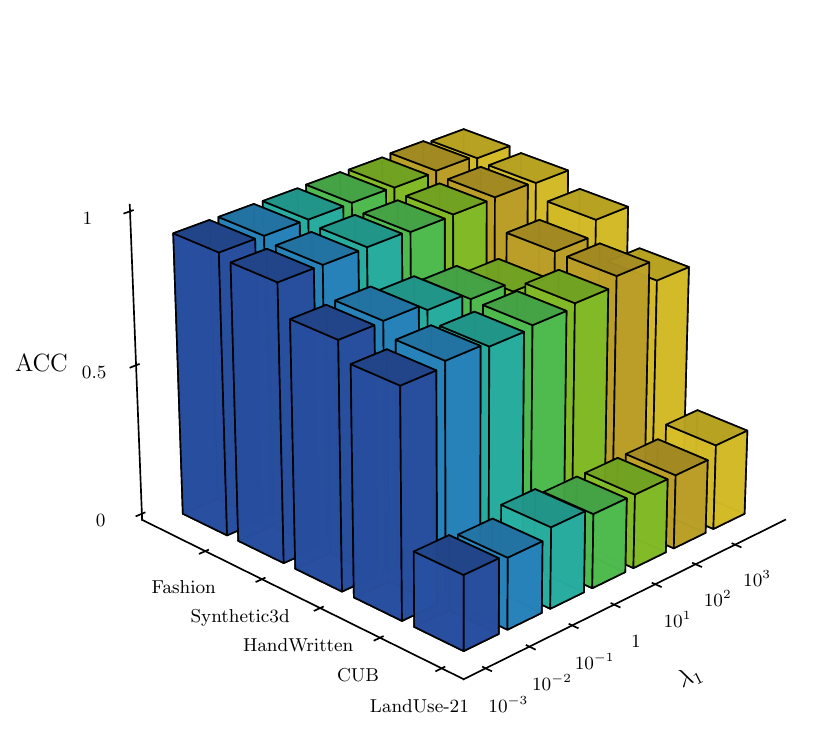}
        \caption{ACC vs. $\lambda_1$}
        \label{fig:lambda1}
    \end{subfigure}%
    \hfill
    \begin{subfigure}[b]{0.25\textwidth}
        \centering
        \includegraphics[width=\textwidth]{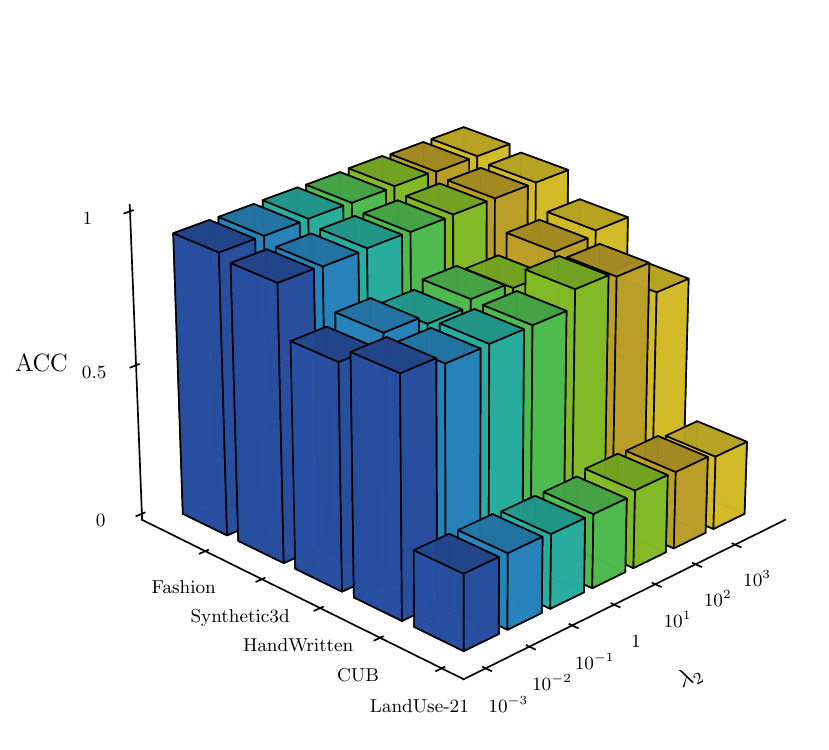}
        \caption{ACC vs. $\lambda_2$}
        \label{fig:nmi_lambda2}
    \end{subfigure}%
    \hfill
    \begin{subfigure}[b]{0.25\textwidth}
        \centering
        \includegraphics[width=\textwidth]{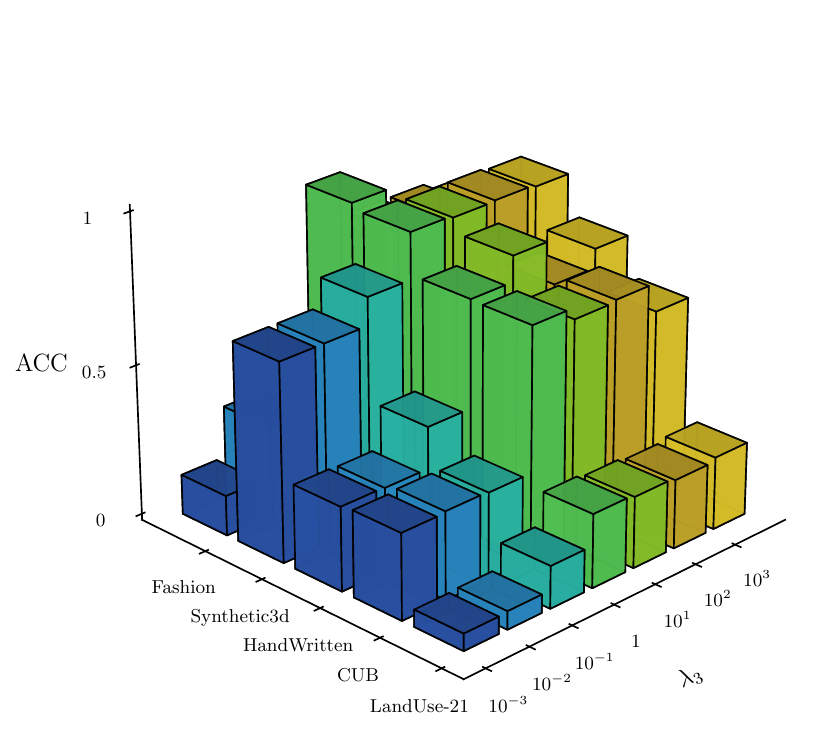}
        \caption{ACC vs. $\lambda_3$}
        \label{fig:lambda3}
    \end{subfigure}%
    \hfill
    \begin{subfigure}[b]{0.25\textwidth}
        \centering
        \includegraphics[width=\textwidth]{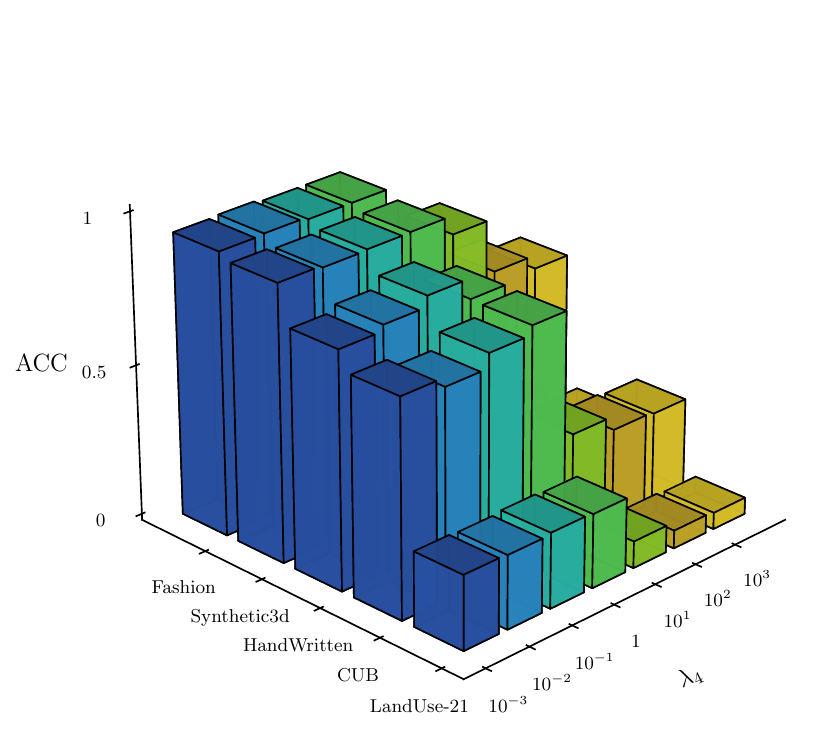}
        \caption{ACC vs. $\lambda_4$}
        \label{fig:lambda4}
    \end{subfigure}
    
    \caption{\textbf{Hyperparameter analysis of the loss coefficients on five datasets with missing rate $\tau=0.3$.}
(a) $\lambda_1$ for the flow prediction loss $\mathcal{L}_{\text{flow-pred}}$;
(b) $\lambda_2$ for the completion consistency loss $\mathcal{L}_{\text{flow-comp}}$;
(c) $\lambda_3$ for the cluster-level contrastive loss $\mathcal{L}_{\text{cluster}}$;
(d) $\lambda_4$ for the entropy-based alignment loss $\mathcal{L}_{\text{entropy}}$.
}
    \label{fig:parameter_analysis}
\end{figure}

\subsection{Analysis}
\label{sec:analysis}

\mypara{Hyperparameters Analysis.}
We analyze the sensitivity of the loss coefficients $\lambda_1$, $\lambda_2$, $\lambda_3$, and $\lambda_4$ under the missing rate $\tau=0.3$, as shown in Fig.~\ref{fig:parameter_analysis}. 
The performance remains stable across a wide range of $\lambda_1$ and $\lambda_2$, corresponding to the flow prediction loss $\mathcal{L}_{\text{flow-pred}}$ and the completion consistency loss $\mathcal{L}_{\text{flow-comp}}$. 
In contrast, the method is more sensitive to $\lambda_3$ and $\lambda_4$, which control the cluster-level contrastive loss $\mathcal{L}_{\text{cluster}}$ and the entropy-based alignment loss $\mathcal{L}_{\text{entropy}}$. 
This observation is consistent with our formulation: alignment losses enforce cross-view assignment consistency and stabilize probability predictions, which are required by the assumptions underlying the flow-matching formulation.

\mypara{Visualization Analysis.}
We visualize the learned representations on the HandWritten dataset using t-SNE under the missing rate $\tau=0.3$, as shown in Fig.~\ref{fig:visual}. 
Compared with the scattered and mixed raw features, the representations learned by our framework form clearly separated clusters. 
This improvement arises because the cluster-level contrastive loss and entropy-based alignment enforce cross-view assignment consistency and stabilize probability predictions, while the flow-matching formulation performs cross-view completion by following the straight-path transport between paired latent representations.

\mypara{Flow-Step Analysis.}
We analyze the influence of the number of flow steps used for cross-view completion. 
As shown in Fig.~\ref{fig:visual}, the clustering performance remains stable across different step numbers. 
This observation is consistent with our theoretical Corollary~\ref{corollary_3} that flow matching achieves cluster separability within any finite number of steps. 
Once the vector field is learned, the straight-path transport learned by the vector field produces stable completion under different finite-step integration settings.

\begin{figure}[t]
    \centering  
    
    \begin{subfigure}[b]{0.64\textwidth}
        \centering
        \includegraphics[height=5cm, width=\linewidth, keepaspectratio]{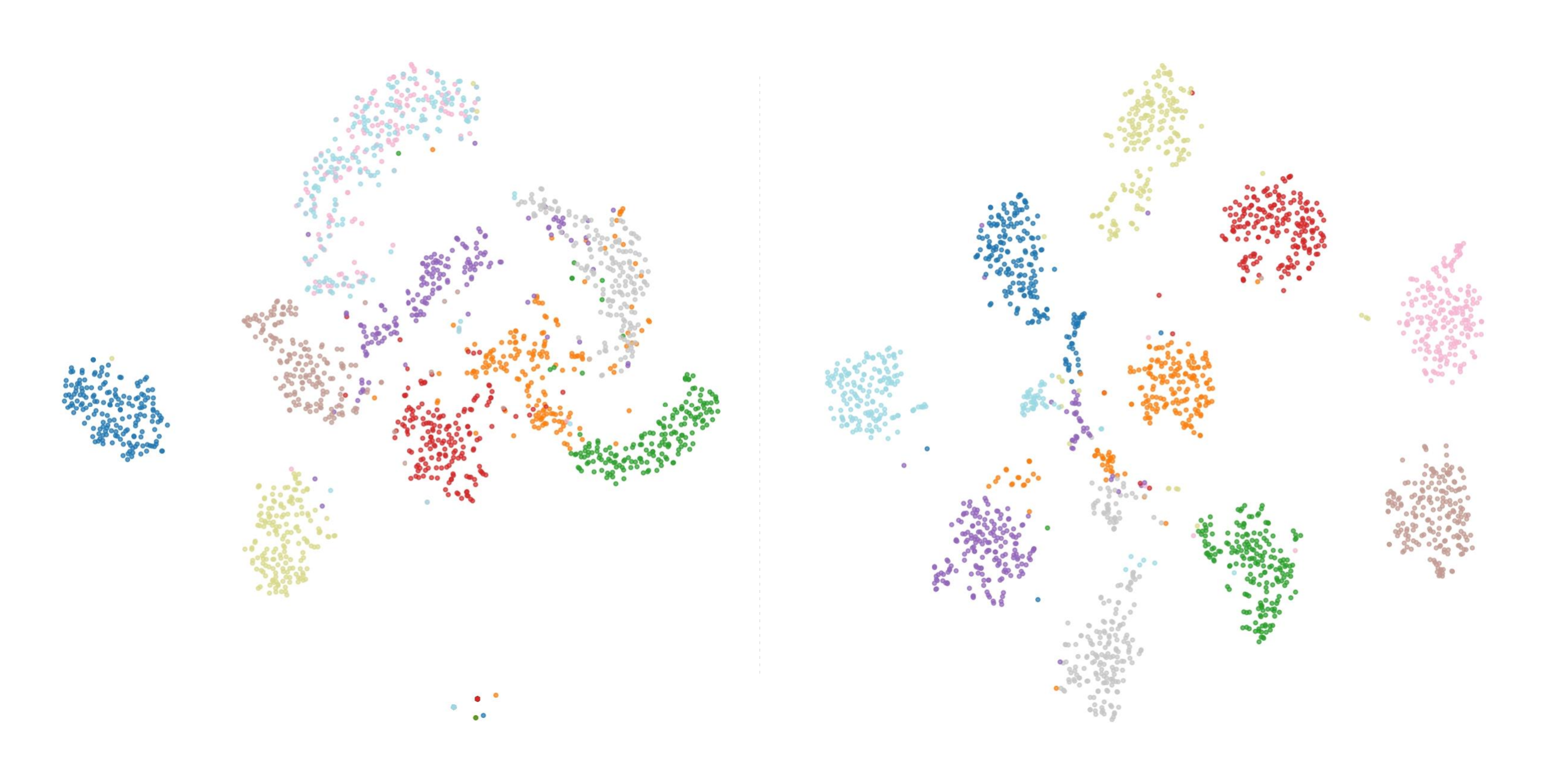}
        \caption{t-SNE Visualization}
        \label{fig:tsne}
    \end{subfigure}
    \begin{subfigure}[b]{0.34\textwidth}
        \centering
        \includegraphics[height=5cm, width=\linewidth, keepaspectratio]{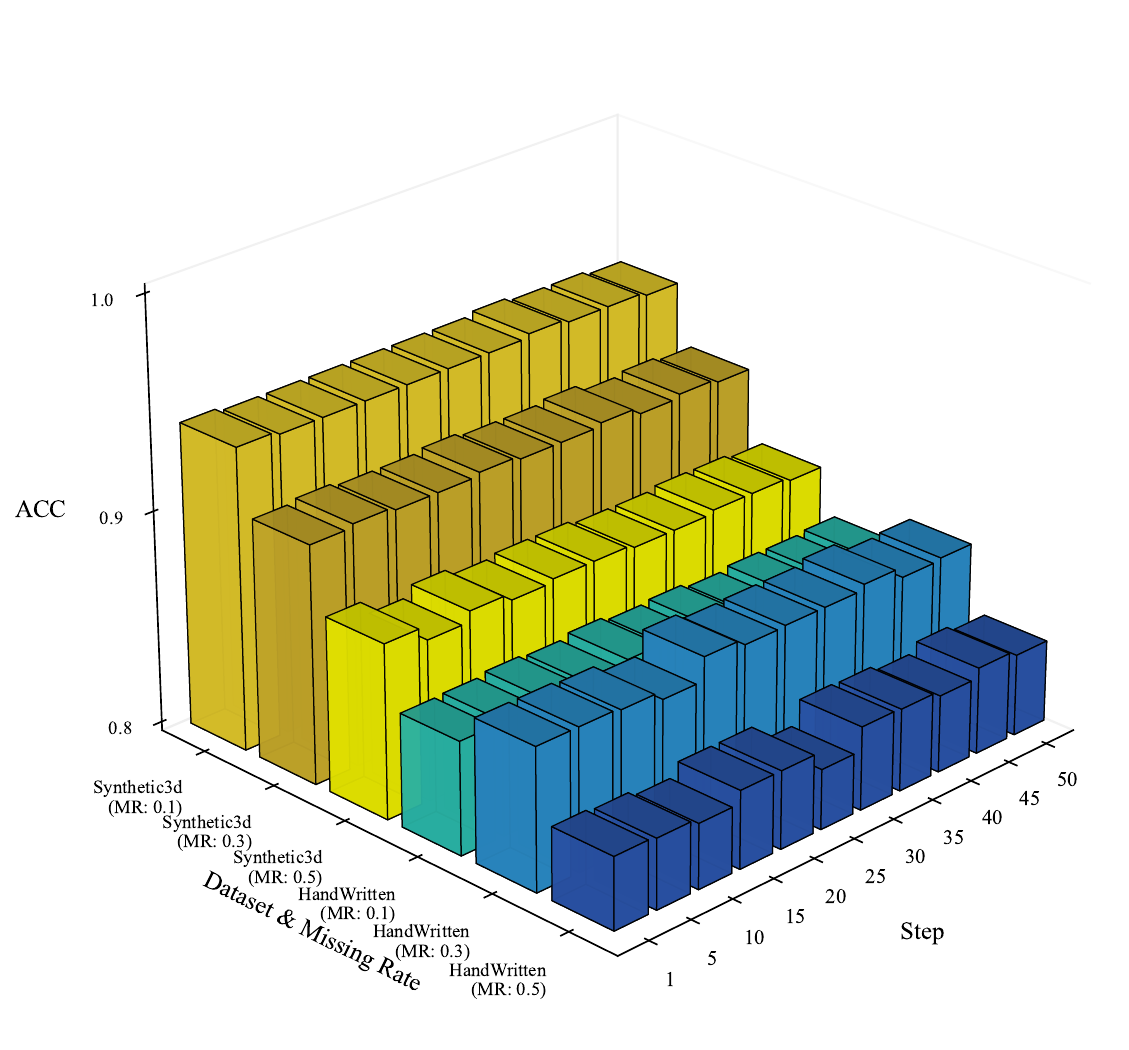}
        \caption{Flow-Step Analysis}
        \label{fig:flow-steps}
    \end{subfigure}
    
    \caption{\textbf{Representation visualization and flow-step analysis.}
(a) t-SNE visualization on the HandWritten dataset with missing rate $\tau=0.3$, showing the raw features (left) and the representations learned by our method (right).
(b) Clustering accuracy under different flow step numbers on the HandWritten and Synthetic3D datasets.}
    \label{fig:visual}
\end{figure}

\section{Conclusion}

We presented a flow-matching framework for incomplete multi-view clustering that formulates missing-view recovery as a straight-path transport problem in the latent representation space. Through theoretical analysis, we showed that flow matching provides two desirable properties for IMVC: shortcut-path transport and finite-step cluster separability. Guided by this analysis, we developed an end-to-end model that integrates within-view reconstruction, cross-view alignment, and flow-based completion. Extensive experiments on multiple benchmarks demonstrate that the proposed approach consistently achieves strong clustering performance and remains robust under high missing rates.

\bibliographystyle{splncs04}
\bibliography{main}

\begin{thebibliography}{10}
\providecommand{\url}[1]{\texttt{#1}}
\providecommand{\urlprefix}{URL }
\providecommand{\doi}[1]{https://doi.org/#1}

\bibitem{asuncion2007uci}
Asuncion, A., Newman, D., et~al.: Uci machine learning repository (2007)

\bibitem{bickel2004multi}
Bickel, S., Scheffer, T.: Multi-view clustering. In: Icdm. vol.~4, pp. 19--26 (2004)

\bibitem{cai2023realize}
Cai, H., Huang, W., Yang, S., Ding, S., Zhang, Y., Hu, B., Zhang, F., Cheung, Y.M.: Realize generative yet complete latent representation for incomplete multi-view learning. IEEE Trans. Pattern Anal. Mach. Intell.  \textbf{46}(5),  3637--3652 (2023)

\bibitem{chao2024incomplete}
Chao, G., Jiang, Y., Chu, D.: Incomplete contrastive multi-view clustering with high-confidence guiding. In: AAAI. vol.~38, pp. 11221--11229 (2024)

\bibitem{dao2023flow}
Dao, Q., Phung, H., Nguyen, B., Tran, A.: Flow matching in latent space. arXiv preprint arXiv:2307.08698  (2023)

\bibitem{dhariwal2021diffusion}
Dhariwal, P., Nichol, A.: Diffusion models beat gans on image synthesis. Adv. Neural Inform. Process. Syst.  \textbf{34},  8780--8794 (2021)

\bibitem{feng2024partial}
Feng, W., Sheng, G., Wang, Q., Gao, Q., Tao, Z., Dong, B.: Partial multi-view clustering via self-supervised network. In: AAAI. vol.~38, pp. 11988--11995 (2024)

\bibitem{gao2016incomplete}
Gao, H., Peng, Y., Jian, S.: Incomplete multi-view clustering. In: ACM Int. Conf. Multimedia. pp. 245--255. Springer (2016)

\bibitem{gui2021review}
Gui, J., Sun, Z., Wen, Y., Tao, D., Ye, J.: A review on generative adversarial networks: Algorithms, theory, and applications. IEEE transactions on knowledge and data engineering  \textbf{35}(4),  3313--3332 (2021)

\bibitem{ho2020denoising}
Ho, J., Jain, A., Abbeel, P.: Denoising diffusion probabilistic models. Adv. Neural Inform. Process. Syst.  \textbf{33},  6840--6851 (2020)

\bibitem{huang2024flow}
Huang, Z., Geng, Z., Luo, W., Qi, G.j.: Flow generator matching. arXiv preprint arXiv:2410.19310  (2024)

\bibitem{jin2023deep}
Jin, J., Wang, S., Dong, Z., Liu, X., Zhu, E.: Deep incomplete multi-view clustering with cross-view partial sample and prototype alignment. In: IEEE Conf. Comput. Vis. Pattern Recog. pp. 11600--11609 (2023)

\bibitem{kingma2013auto}
Kingma, D.P., Welling, M.: Auto-encoding variational bayes. arXiv preprint arXiv:1312.6114  (2013)

\bibitem{klein2023equivariant}
Klein, L., Kr{\"a}mer, A., No{\'e}, F.: Equivariant flow matching. Adv. Neural Inform. Process. Syst.  \textbf{36},  59886--59910 (2023)

\bibitem{kumar2011co}
Kumar, A., Rai, P., Daume, H.: Co-regularized multi-view spectral clustering. Adv. Neural Inform. Process. Syst.  \textbf{24} (2011)

\bibitem{labs2025flux}
Labs, B.F., Batifol, S., Blattmann, A., Boesel, F., Consul, S., Diagne, C., Dockhorn, T., English, J., English, Z., Esser, P., et~al.: Flux. 1 kontext: Flow matching for in-context image generation and editing in latent space. arXiv preprint arXiv:2506.15742  (2025)

\bibitem{li2023incomplete}
Li, H., Li, Y., Yang, M., Hu, P., Peng, D., Peng, X.: Incomplete multi-view clustering via prototype-based imputation. arXiv preprint arXiv:2301.11045  (2023)

\bibitem{li2025community}
Li, H., Lin, Y., Hu, P., Yang, M., Peng, X.: Community-aware multi-view representation learning with incomplete information. IEEE Trans. Pattern Anal. Mach. Intell.  (2025)

\bibitem{lin2023consistent}
Lin, R., Du, S., Wang, S., Guo, W.: Consistent graph embedding network with optimal transport for incomplete multi-view clustering. Information Sciences  \textbf{647},  119418 (2023)

\bibitem{lin2022dual}
Lin, Y., Gou, Y., Liu, X., Bai, J., Lv, J., Peng, X.: Dual contrastive prediction for incomplete multi-view representation learning. IEEE Trans. Pattern Anal. Mach. Intell.  \textbf{45}(4),  4447--4461 (2022)

\bibitem{lin2021completer}
Lin, Y., Gou, Y., Liu, Z., Li, B., Lv, J., Peng, X.: Completer: Incomplete multi-view clustering via contrastive prediction. In: IEEE Conf. Comput. Vis. Pattern Recog. pp. 11174--11183 (2021)

\bibitem{lipman2022flow}
Lipman, Y., Chen, R.T., Ben-Hamu, H., Nickel, M., Le, M.: Flow matching for generative modeling. arXiv preprint arXiv:2210.02747  (2022)

\bibitem{liu2022flow}
Liu, X., Gong, C., Liu, Q.: Flow straight and fast: Learning to generate and transfer data with rectified flow. arXiv preprint arXiv:2209.03003  (2022)

\bibitem{liu2020efficient}
Liu, X., Li, M., Tang, C., Xia, J., Xiong, J., Liu, L., Kloft, M., Zhu, E.: Efficient and effective regularized incomplete multi-view clustering. IEEE Trans. Pattern Anal. Mach. Intell.  \textbf{43}(8),  2634--2646 (2020)

\bibitem{liu2019efficient}
Liu, X., Zhu, X., Li, M., Tang, C., Zhu, E., Yin, J., Gao, W.: Efficient and effective incomplete multi-view clustering. In: AAAI. vol.~33, pp. 4392--4399 (2019)

\bibitem{oord2018representation}
Oord, A.v.d., Li, Y., Vinyals, O.: Representation learning with contrastive predictive coding. arXiv preprint arXiv:1807.03748  (2018)

\bibitem{podell2023sdxl}
Podell, D., English, Z., Lacey, K., Blattmann, A., Dockhorn, T., M{\"u}ller, J., Penna, J., Rombach, R.: Sdxl: Improving latent diffusion models for high-resolution image synthesis. arXiv preprint arXiv:2307.01952  (2023)

\bibitem{schusterbauer2025diff2flow}
Schusterbauer, J., Gui, M., Fundel, F., Ommer, B.: Diff2flow: Training flow matching models via diffusion model alignment. In: IEEE Conf. Comput. Vis. Pattern Recog. pp. 28347--28357 (2025)

\bibitem{songdenoising}
Song, J., Meng, C., Ermon, S.: Denoising diffusion implicit models. In: Int. Conf. Learn. Represent. (2021)

\bibitem{song2020score}
Song, Y., Sohl-Dickstein, J., Kingma, D.P., Kumar, A., Ermon, S., Poole, B.: Score-based generative modeling through stochastic differential equations. arXiv preprint arXiv:2011.13456  (2020)

\bibitem{tang2022deep}
Tang, H., Liu, Y.: Deep safe incomplete multi-view clustering: Theorem and algorithm. In: Int. Conf. Mach. Learn. pp. 21090--21110. PMLR (2022)

\bibitem{wah2011caltech}
Wah, C., Branson, S., Welinder, P., Perona, P., Belongie, S., et~al.: The caltech-ucsd birds-200-2011 dataset. Tech. rep., Technical Report CNS-TR-2011-001, California Institute of Technology (2011)

\bibitem{wang2023self}
Wang, J., Xu, Z., Yang, X., Guo, D., Liu, L.: Self-supervised image clustering from multiple incomplete views via constrastive complementary generation. IET Computer Vision  \textbf{17}(2),  189--202 (2023)

\bibitem{wang2020generative}
Wang, Q., Ding, Z., Tao, Z., Gao, Q., Fu, Y.: Generative partial multi-view clustering. arXiv preprint arXiv:2003.13088  (2020)

\bibitem{wang2022incomplete}
Wang, Y., Chang, D., Fu, Z., Wen, J., Zhao, Y.: Incomplete multiview clustering via cross-view relation transfer. IEEE Trans. Circuit Syst. Video Technol.  \textbf{33}(1),  367--378 (2022)

\bibitem{wei2020deep}
Wei, S., Wang, J., Yu, G., Domeniconi, C., Zhang, X.: Deep incomplete multi-view multiple clusterings. In: 2020 IEEE International Conference on Data Mining (ICDM). pp. 651--660. IEEE (2020)

\bibitem{wen2024diffusion}
Wen, J., Deng, S., Wong, W., Chao, G., Huang, C., Fei, L., Xu, Y.: Diffusion-based missing-view generation with the application on incomplete multi-view clustering. In: Int. Conf. Mach. Learn. (2024)

\bibitem{wen2021structural}
Wen, J., Wu, Z., Zhang, Z., Fei, L., Zhang, B., Xu, Y.: Structural deep incomplete multi-view clustering network. In: Proceedings of the 30th ACM international conference on information \& knowledge management. pp. 3538--3542 (2021)

\bibitem{wen2023graph}
Wen, J., Xu, G., Tang, Z., Wang, W., Fei, L., Xu, Y.: Graph regularized and feature aware matrix factorization for robust incomplete multi-view clustering. IEEE Trans. Circuit Syst. Video Technol.  \textbf{34}(5),  3728--3741 (2023)

\bibitem{wen2018incomplete}
Wen, J., Xu, Y., Liu, H.: Incomplete multiview spectral clustering with adaptive graph learning. IEEE transactions on cybernetics  \textbf{50}(4),  1418--1429 (2018)

\bibitem{xiao2017fashion}
Xiao, H., Rasul, K., Vollgraf, R.: Fashion-mnist: a novel image dataset for benchmarking machine learning algorithms. arXiv preprint arXiv:1708.07747  (2017)

\bibitem{xu2024deep}
Xu, G., Wen, J., Liu, C., Hu, B., Liu, Y., Fei, L., Wang, W.: Deep variational incomplete multi-view clustering: Exploring shared clustering structures. In: AAAI. vol.~38, pp. 16147--16155 (2024)

\bibitem{xu2023adaptive}
Xu, J., Li, C., Peng, L., Ren, Y., Shi, X., Shen, H.T., Zhu, X.: Adaptive feature projection with distribution alignment for deep incomplete multi-view clustering. IEEE Trans. Image Process.  \textbf{32},  1354--1366 (2023)

\bibitem{xu2022deep}
Xu, J., Li, C., Ren, Y., Peng, L., Mo, Y., Shi, X., Zhu, X.: Deep incomplete multi-view clustering via mining cluster complementarity. In: AAAI. vol.~36, pp. 8761--8769 (2022)

\bibitem{yan2023gcfagg}
Yan, W., Zhang, Y., Lv, C., Tang, C., Yue, G., Liao, L., Lin, W.: Gcfagg: Global and cross-view feature aggregation for multi-view clustering. In: IEEE Conf. Comput. Vis. Pattern Recog. pp. 19863--19872 (2023)

\bibitem{yang2010bag}
Yang, Y., Newsam, S.: Bag-of-visual-words and spatial extensions for land-use classification. In: Proceedings of the 18th SIGSPATIAL international conference on advances in geographic information systems. pp. 270--279 (2010)

\bibitem{yin2021incomplete}
Yin, J., Sun, S.: Incomplete multi-view clustering with reconstructed views. IEEE Transactions on Knowledge and Data Engineering  \textbf{35}(3),  2671--2682 (2021)

\bibitem{zhang2019cpm}
Zhang, C., Han, Z., Fu, H., Zhou, J.T., Hu, Q., et~al.: Cpm-nets: Cross partial multi-view networks. Adv. Neural Inform. Process. Syst.  \textbf{32} (2019)

\bibitem{zhang2023robust}
Zhang, C., Wei, J., Wang, B., Li, Z., Chen, C., Li, H.: Robust spectral embedding completion based incomplete multi-view clustering. In: ACM Int. Conf. Multimedia. pp. 300--308 (2023)

\bibitem{zhang2025incomplete}
Zhang, Y., Lin, Y., Yan, W., Yao, L., Wan, X., Li, G., Zhang, C., Ke, G., Xu, J.: Incomplete multi-view clustering via diffusion contrastive generation. In: AAAI. vol.~39, pp. 22650--22658 (2025)

\bibitem{zhao2017multi}
Zhao, H., Ding, Z., Fu, Y.: Multi-view clustering via deep matrix factorization. In: AAAI. vol.~31 (2017)

\end{thebibliography}

\clearpage
\renewcommand{\thesection}{\Alph{section}}
\renewcommand{\thesubsection}{\Alph{section}.\arabic{subsection}}
\renewcommand{\thesubsubsection}{\Alph{section}.\arabic{subsection}.\arabic{subsubsection}}

\edef\savedequation{\the\value{equation}}  
\edef\savedtable{\the\value{table}}        
\edef\savedsection{\the\value{section}}    
\edef\savedfigure{\the\value{figure}}

\appendix
\renewcommand{\theHsection}{appendix.\Alph{section}}
\renewcommand{\theHsubsection}{appendix.\Alph{section}.\arabic{subsection}}
\renewcommand{\theHsubsubsection}{appendix.\Alph{section}.\arabic{subsection}.\arabic{subsubsection}}
\section*{Appendix}

\setcounter{section}{0}

\section{Limitations of Conditional Diffusion Models on IMVC}\label{sec:mathematical-diffusion}

Here, we establish a formal foundation for the comparative theoretical analysis between the diffusion model and flow matching.
Specifically, we first augment the problem setup in \secref{sec:theoretical-analysis} in the context of conditional diffusion models (\secref{sec:problem-setup}).
Then, based on the assumptions we have enumerated in \secref{sec:theoretical-analysis}, we complete the full mathematical formulation of the continuous-time conditional diffusion model, including both the forward process and the reverse denoising process on the task of IMVC (\secref{sec:conditional-diffusion-process}).
In~\secref{sec:conditional-diffusion-process},  we summarize three properties of conditional diffusion models, which lead to two limitations of such models under the IMVC task setup (\secref{sec:diffusion-limitations}).

This section complements \secref{sec:theoretical-analysis}.

\subsection{Additional Problem Setup}\label{sec:problem-setup}

This setup augments the problem setup we introduced in \secref{sec:problem-setup-short} in the context of the conditional diffusion model.

The diffusion process considered in subsequent analysis is defined over continuous time $t \in [0,T]$, where $t=0$ corresponds to the clean representation space and $t=T$ corresponds to pure Gaussian noise. Let $g(t) \in \mathbb{R}_+$ denote the diffusion coefficient, and define the accumulated noise variance as
$\sigma_t^2 = \int_0^t g^2(s)\,ds$,
with $g(0)=0$, $g(T)>0$.

\subsection{Mathematical Formulation of Conditional Diffusion 
Models in IMVC}\label{sec:conditional-diffusion-process}

Previous work~\cite{zhang2025incomplete} has established that the conditional diffusion model can be adopted for the task of IMVC~\footnote{We refer readers to~\secref{sec:proof-of-diffusion}, Thm.~\ref{thm:diffusion_clustering} for the proof of diffusion’s theoretical applicability to the task, complementing our discussion in this section, as well as the work of Zhang et al.~\cite{zhang2025incomplete}}.
However, we argue that its application to the task is suboptimal due to the intrinsic theoretical limitations.

Below, we present a complete mathematical formulation of the conditional diffusion models, where we establish three key properties, by analyzing the forward and reverse processes.
The properties serve as the theoretical foundation for the analysis of its limitations in IMVC in~\secref{sec:diffusion-limitations}.

\paragraph{Forward Noising Process}
The forward noising process of the diffusion model is the Ornstein--Uhlenbeck (OU) SDE:
\begin{equation}\label{eqn:property1_eq1}
d \boldsymbol{x}_t = -\frac{1}{2} g^2(t) \boldsymbol{x}_t \, dt + g(t)\, d \boldsymbol{w}_t, \quad \boldsymbol{x}_0 = \boldsymbol{z}_2;
\end{equation}
where $\boldsymbol{x}_t \in Z$ denotes the noised representation at time $t$, 
$\boldsymbol{w}_t$ is a $d$-dimensional standard Wiener process (\ie the stochastic noise source), 
and $\boldsymbol{z}_2$ is the clean representation of view 2, as the missing view that shall be recovered.

By linearity of Gaussian processes, \eqnref{eqn:property1_eq1} admits an explicit solution:
\begin{equation}\label{eqn:property1_eq2}
\boldsymbol{x}_t = \boldsymbol{z}_2 \cdot e^{-\frac{1}{2}\sigma_t^2}
+ \int_0^t g(s)\, e^{-\frac{1}{2}\int_s^t g^2(r)\,dr}\, d \boldsymbol{w}_s.
\end{equation}
Its marginal distribution is Gaussian:
\begin{equation}\label{eqn:property1_pre1}
p_t(\boldsymbol{x}_t\mid \boldsymbol{z}_2)=\mathcal{N}\!\left(\boldsymbol{x}_t;\alpha_t z_2,\sigma_t^2 I_d\right),\quad 
\alpha_t=e^{-\frac{1}{2}\sigma_t^2};
\end{equation}
where $\alpha_t \in (0,1]$ is the signal attenuation coefficient, satisfying 
$\alpha_0 = 1$ and 
$\alpha_T \to 0$; 
$I_d$ denotes the $d$-dimensional identity matrix.

Hence, in the high-noise limit $t\to T$ (\ie $\alpha_T\to 0$), the terminal distribution becomes a global Gaussian independent of the underlying cluster set:
\begin{equation}\label{eqn:property1}
p_T(\boldsymbol{x}_T\mid \boldsymbol{z}_2)\to \mathcal{N}\!\left(\boldsymbol{x}_T;0,\sigma_T^2 I_d\right).
\end{equation}
As such, we arrive at
\begin{myproperty}\label{prop:1}
Regardless of which cluster set $\bmz_2$ belongs to, the diffusion endpoint at $t\to T$ (\ie $\alpha_T\to 0$) is cluster-agnostic noise and cluster identity is fully obscured.
\end{myproperty}

\paragraph{Reverse Denoising Process}
In IMVC, diffusion model recovers the clean representation of $\boldsymbol{z}_2$ at $t=0$ conditioned on the observed-view representation $\boldsymbol{z}_1$ by integrating a reverse-time conditional SDE:
\begin{equation}\label{eqn:property2_pre1}
    d\boldsymbol{x}_t=\Big[g^2(t)s^*(\boldsymbol{x}_t,t,\boldsymbol{z}_1)-\frac{1}{2}g^2(t)\boldsymbol{x}_t\Big]dt+g(t)d\boldsymbol{\bar w}_t,\quad
\boldsymbol{x}_T\sim \mathcal{N}(0,\sigma_T^2 I_d).
\end{equation}
Here $\boldsymbol{\bar w}_t$ is a reverse-time Wiener process (independent of the forward one) and constitutes a non-removable source of randomness; $s^*(\boldsymbol{x}_t,t,\boldsymbol{z}_1)=-\nabla_{\boldsymbol{x}_t}\log p_t(\boldsymbol{x}_t\mid \boldsymbol{z}_1)$ is the true conditional score.
From~\eqnref{eqn:property1} and~\ref{eqn:property2_pre1}, we derive
\begin{myproperty}\label{prop:2}
The reverse process starts from a global Gaussian that is geometrically unrelated to the cluster set of $\boldsymbol{z}_1$:
\begin{equation}\label{eqn:property2}
    p_T(\boldsymbol{x}_T\mid \boldsymbol{z}_1)=\mathcal{N}(0,\sigma_T^2 I_d).
\end{equation}
This stochasticity and cluster-agnostic start are intrinsic to the diffusion mechanism and cannot be eliminated by training or tuning.  
\end{myproperty}

\paragraph{Cluster Separability Condition of the Reverse Process}
A necessary condition for diffusion-based completion to preserve clustering is that the supports of noised distributions from different cluster sets do not overlap.
For any $\boldsymbol{z}_2^i\in M_i$ and $\boldsymbol{z}_2^j\in M_j$ with $i\neq j$, the separability requirement is
\begin{equation}\label{eqn:property3_pre}
    \mathrm{supp}\!\left(p_t(\boldsymbol{x}_t\mid \boldsymbol{z}_2^i)\right)\cap \mathrm{supp}\!\left(p_t(\boldsymbol{x}_t\mid \boldsymbol{z}_2^j)\right)=\emptyset.
\end{equation}
Using the Gaussian form in~\eqnref{eqn:property1_pre1}, this condition holds iff
\begin{equation}\label{eqn:property3_pre1}
\sigma_t<\frac{\delta}{2},
\end{equation}
where $\delta$ is the minimum inter-cluster distance from Assumption~\ref{assump:separable_representation}. Therefore, we naturally derive
\begin{myproperty}\label{prop:3}
Diffusion-induced cluster separability is guaranteed only when the accumulated noise variance is sufficiently small --- \ie essentially in the asymptotic regime $t\to 0$. 
\end{myproperty}

\subsection{Limitations of Diffusion Models in IMVC}
\label{sec:diffusion-limitations}

The mathematical formulation and three properties of conditional diffusion models lead to two limitations for these models when applied to the IMVC task.
We note that these limitations cannot be resolved through engineering efforts, such as hyperparameter tuning or architectural designs.
Below, we introduce these limitations while connecting them to the properties introduced in~\secref{sec:conditional-diffusion-process}.

\begin{limitation}[Cluster-Agnostic Gaussian Initialization]
From~\eqnref{eqn:property1} and~\ref{eqn:property2} in Property~\ref{prop:1} and~\ref{prop:2}, the reverse process of the diffusion model starts from a fixed initialization $\boldsymbol{x}_T \sim \mathcal{N}(0, \sigma_T^2 I_d)$.
This starting point is a global Gaussian noise distribution that has no association with the observed view $\boldsymbol{z}_1$ or the cluster set to which it belongs.

As such, we point out two direct consequences. 
The first is path redundancy, \ie the reverse process must complete the full evolution from the global noise space to the representation space and finally to the cluster set.
Even if the observed view and the completion target are highly aligned (\ie $\|\boldsymbol{z}_1 - \boldsymbol{z}_2\| \to 0$), the process cannot exploit the structural information of view consistency to form a shortcut path.
A substantial portion of computation is therefore consumed in noise attenuation rather than view completion and clustering refinement.

The second consequence is geometric irrelevance.
Since the noise initialization contains no cluster geometric information, the early stage of the reverse process must extract cluster structure from noise without any geometric prior, resulting in low efficiency for IMVC.

\end{limitation}

\begin{limitation}[Asymptotic Cluster Separability]
From Property~\ref{prop:3}, \eqnref{eqn:property3_pre1}, the cluster separability of the diffusion model holds only when $\sigma_t < \delta/2$ (\ie $t \to 0$).
This strong asymptotic condition is difficult to satisfy under finite-step inference in IMVC, where we note two implications.

The first being the significant degradattion of separability under finite steps, \ie when the minimum inter-cluster distance $\delta$ is small (which is common in real-world data), it is possible that $\sigma_t \not< \delta/2$ within finite steps, resulting in overlap between the supports of noised distributions from different clusters, and samples are prone to cross-cluster drift.
This will reduce the clustering accuracy.

And second, clustering becomes a byproduct of noise attenuation — the cluster structure in diffusion models emerges only when the noise level becomes sufficiently small, and most reverse-time steps are devoted to noise reduction rather than the core objective of view completion and fine-grained clustering adjustment, leading to low efficiency.
\end{limitation}

The two limitations discussed above collectively demonstrate that diffusion model is structurally misaligned with the task requirements of IMVC.

\section{Proof of Clustering Properties for Diffusion and Flow Matching}\label{sec:proof-of-clustering-properties}

To fill in the gap of existing work~\cite{zhang2025incomplete}, which has only empirically adopted diffusion models for end-to-end IMVC without theoretical analysis, we provide proof that the reverse denoising process of conditional diffusion models is intrinsically a clustering process (\secref{sec:proof-of-diffusion}).

\subsection{Clustering Property of the Reverse Denoising Process in Conditional Diffusion Models}\label{sec:proof-of-diffusion}

Based on the complete mathematical formulation and the properties of the conditional diffusion model established above (\secref{sec:conditional-diffusion-process}), we prove that the reverse denoising process of the conditional diffusion model is intrinsically a clustering process.
In essence, starting from pure Gaussian noise at $t = T$, the reverse SDE evolves such that, as $t \to 0$, the sample trajectory converges to the cluster set associated with the observed view with probability one.

This proof completes our theoretical analysis, while complementing previous work~\cite{zhang2025incomplete} that has empirically adopted diffusion models for end-to-end IMVC.

\begin{theorem}[Clustering Convergence of Conditional Diffusion Models]
\label{thm:diffusion_clustering}
Under the Assumptions~\ref{assump:separable_representation}-\ref{assump:perfect_approx}, consider the reverse denoising process of the conditional diffusion model defined by the reverse SDE (\eqnref{eqn:property2_pre1}). Starting from pure Gaussian noise
$\boldsymbol{x}_T \sim \mathcal{N}\left(0, \sigma_T^2 I_d\right)$,
as $t \to 0$, the noisy representation $\boldsymbol{x}_t$ converges almost surely to the cluster set associated with the observed view $\boldsymbol{z}_1$, namely,
\begin{equation}
\mathbb{P}\!\left( \boldsymbol{x}_t \in \mathcal{S}_{c(\boldsymbol{z}_1)} \mid \boldsymbol{z}_1 \right) \to 1, 
\qquad t \to 0.
\label{eq:cluster_convergence_probability}
\end{equation}
\end{theorem}

\begin{proof}
The proof proceeds in three steps, each directly grounded in the diffusion equations (\eqnref{eqn:property1_eq1}--\ref{eqn:property3_pre1}) and the stated assumptions.

\paragraph{Step 1: Cluster Concentration of the Conditional Prior.}

By the view-consistency assumption (Assumption~\ref{assump:view_consistency}), for a given observed view representation $\boldsymbol{z}_1$, the target representation $\boldsymbol{z}_2$ must lie on the same cluster. Therefore, the support of the conditional prior distribution satisfies
\begin{equation}
\mathrm{supp}\big(p_0(\boldsymbol{z}_2 \mid \boldsymbol{z}_1)\big) \subseteq \mathcal{S}_{c(\boldsymbol{z}_1)}.
\label{eq:conditional_prior_support}
\end{equation}

This implies that the conditional prior assigns nonzero probability only within the unique cluster set corresponding to $\boldsymbol{z}_1$, and thus no probability mass is distributed across different or disjoint clusters.

\paragraph{Step 2: Cluster Contraction of the Conditional Posterior.}

By Bayes' rule, the conditional posterior distribution of the noisy representation at time $t$ is
\begin{equation}
p_t(\boldsymbol{x}_t \mid \boldsymbol{z}_1) 
= \int_{\mathcal{Z}} p_t(\boldsymbol{x}_t \mid \boldsymbol{z}_2)\, p_0(\boldsymbol{z}_2 \mid \boldsymbol{z}_1)\, d\boldsymbol{z}_2.
\label{eq:posterior_decomposition}
\end{equation}

From the Gaussian marginal distribution of the forward process (\eqnref{eqn:property1_pre1}),
we know that as $t \to 0$, the accumulated noise variance satisfies $\sigma_t \to 0$. Consequently, each conditional Gaussian distribution $p_t\left(\boldsymbol{x}_t | \boldsymbol{z}_2 \right)$ contracts to a small neighborhood of the clean representation $\boldsymbol{z}_2$.

By the cluster separability condition (\eqnref{eqn:property3_pre1}), when $\sigma_t < \delta/2$, the supports of the noisy conditional distributions $p_t(\boldsymbol{x}_t \mid \boldsymbol{z}_2)$ corresponding to different cluster sets are disjoint. 
Combining this with \eqref{eq:conditional_prior_support}, the support of the posterior distribution $p_t(\boldsymbol{x}_t \mid \boldsymbol{z}_1)$ is contained entirely within a neighborhood of $\mathcal{S}_{c(\boldsymbol{z}_1)}$.

\paragraph{Step 3: Asymptotic Convergence of Cluster Probability.}

As $t \to 0$, $\delta \to 0$, the Gaussian distribution converges to a Dirac measure centered at $\boldsymbol{z}_2$:
\begin{equation}
p_t(\boldsymbol{x}_t \mid \boldsymbol{z}_2) 
\to \delta(\boldsymbol{x}_t - \boldsymbol{z}_2).
\label{eq:dirac_limit}
\end{equation}

Substituting this into \eqref{eq:posterior_decomposition}, we obtain
\begin{equation}
\lim_{t \to 0} p_t(\boldsymbol{x}_t \mid \boldsymbol{z}_1)
= p_0(\boldsymbol{x}_0 \mid \boldsymbol{z}_1),
\label{eq:posterior_limit}
\end{equation}
whose support lies entirely within $\mathcal{S}_{c(\boldsymbol{z}_1)}$.

Therefore, integrating the posterior distribution over the cluster set yields
\begin{equation}
\mathbb{P}\!\left( \boldsymbol{x}_t \in \mathcal{S}_{c(\boldsymbol{z}_1)} \mid \boldsymbol{z}_1 \right)
=
\int_{\mathcal{S}_{c(\boldsymbol{z}_1)}} 
p_t(\boldsymbol{x}_t \mid \boldsymbol{z}_1)\, d\boldsymbol{x}_t
\to 1,
\qquad t \to 0.
\label{eq:probability_limit}
\end{equation}

This completes the proof.
\end{proof}

We note that Thm.~\ref{thm:diffusion_clustering} establishes clustering convergence only in the asymptotic regime $t \to 0$.
However, the diffusion equations reveal three intrinsic structural properties that constitute theoretical limitations of diffusion models in the IMVC setting, as discussed in~\secref{sec:diffusion-limitations}.

\section{Proof of Corollaries on Flow Matching}\label{sec:prove-corollary}
This section complements \secref{sec:benefits-of-flow-matching} by supplying proofs of the corollaries on flowing matching.

\begin{proof}[Proof of Cor.\ref{corollary_2}]
The initial condition of flow matching is $\bmx_0 = \bmz_1$, and since $\bmz_1 \in \mathcal{S}_{c(\bmz_1)}$ (the cluster set of the observed view), the starting point itself contains complete cluster structure information. There is no need to mine the cluster structure from noise, eliminating path redundancy at its source.
    
The linear interpolation path in \eqnref{eqn:flow-1} is the shortest path in Euclidean space connecting $\bmz_1$ and $\bmz_2$, with a path length of $\|\bmz_2 - \bmz_1\|$ and no detours.

\end{proof}

\begin{proof}[Proof of Cor.\ref{corollary_3}]
   The target vector field for flow matching is $\bmv_{\text{target}} = \bmz_2 - \bmz_1$ (\eqnref{eqn:flow-2}). By the view consistency assumption (\secref{sec:theoretical-analysis}), $\bmz_1, \bmz_2 \in \mathcal{S}_{c(\bmz_1)}$. Therefore, the target vector field always points inside the same cluster set, and the vector field $\bmv_\theta(\cdot)$ fitted by the neural network naturally possesses a strong cluster-guiding property.

    By the uniqueness theorem for ODE solutions, for any two initial conditions $\bmz_1^i,\bmz_2^i\in \mathcal{S}_i$ and $\bmz_1^j,\bmz_2^j \in \mathcal{S}_j \ (i \neq j)$ belonging to different cluster sets, their corresponding ODE solution trajectories $\bmx_t^i, \bmx_t^j$ never intersect. That is:
    \begin{align}
    \bmx_t^i \in \mathcal{S}_i, \quad \bmx_t^j \in \mathcal{S}_j,\quad \forall t \in [0,1]\\ \bmx_t^i = (1-t)\bmz_1^i+t\bmz_2^i,\quad \bmx_t^j = (1-t)\bmz_1^j+t\bmz_2^j
    \end{align}

    Due to the non-intersecting property of trajectories from different clusters, flow matching guarantees strict separation of sample representations from different clusters within any finite number of steps. 

\end{proof}

\section{Proof of Reversibility in Flow Matching}
\label{app:integral_proof}

In this section, we provide a proof of the reversibility property of the flow matching ODE framework, which complements our discussion in \secref{sec:preliminary-flow-matching}.

\subsection{Preliminaries of Forward Flow}
First, recall the forward flow matching ODE that evolves a sample from $\bm{z}_1$ to $\bm{z}_2$ over time $t \in [0,1]$:
\begin{equation}
\frac{d\bm{x}_t}{dt} = \bm{v}_\theta(\bm{x}_t, t), \quad \bm{x}_0 = \bm{z}_1, \quad \bm{x}_1 = \bm{z}_2.
\end{equation}
The solution to this ODE is the integral:
\begin{equation}
\bm{x}_1 - \bm{x}_0 = \int_0^1 \bm{v}_\theta(\bm{x}_t, t) dt.
\end{equation}

\subsection{Reverse Flow via Variable Substitution}
To derive the reverse flow (from $\bm{z}_2$ to $\bm{z}_1$), we introduce a reversed time variable $s = 1 - t$. 
When $t = 0$, $s = 1$; when $t = 1$, $s = 0$. 
The differential of $t$ with respect to $s$ is:
\begin{equation}
dt = -ds.
\end{equation}

Define the reversed state variable $\bm{x}_s = \bm{x}_{1-t}$, which satisfies:
\begin{equation}
\bm{x}_s = (1-s)\bm{z}_2 + s\bm{z}_1,
\end{equation}
with initial condition $\bm{x}_0 = \bm{z}_2$ and final condition $\bm{x}_1 = \bm{z}_1$ (consistent with reverse flow).

Compute the time derivative of $\bm{x}_s$ with respect to $s$:
\begin{align}
\frac{d\bm{x}_s}{ds} &= \frac{d\bm{x}_{1-t}}{ds} = \frac{d\bm{x}_{1-t}}{d(1-t)} \cdot \frac{d(1-t)}{ds} \\
&= \bm{v}_\theta(\bm{x}_{1-t}, 1-t) \cdot (-1) \\
&= -\bm{v}_\theta(\bm{x}_{1-t}, 1-t).
\end{align}

By defining the reverse vector field as $\bm{v}_{\theta_R}(\bm{x}_s, s) = -\bm{v}_\theta(\bm{x}_{1-t}, 1-t)$,  the corresponding reverse ODE simplifies to:
\begin{equation}
\frac{d\bm{x}_s}{ds} = \bm{v}_{\theta_R}(\bm{x}_s, s), \quad \bm{x}_0 = \bm{z}_2, \quad \bm{x}_1 = \bm{z}_1.
\end{equation}

\subsection{Verification of Bidirectional Consistency}
Integrate the reverse ODE to verify it recovers the original state:
\begin{align}
\bm{x}_1 - \bm{x}_0 &= \int_0^1 \bm{v}_{\theta_R}(\bm{x}_s, s) ds \\
&= -\int_0^1 \bm{v}_\theta(\bm{x}_{1-t}, 1-t) ds \\
&= \int_1^0 \bm{v}_\theta(\bm{x}_t, t) dt = -\int_0^1 \bm{v}_\theta(\bm{x}_t, t) dt \\
&= \bm{z}_1 - \bm{z}_2.
\end{align}

This confirms that the reverse vector field $-\bm{v}_\theta(\cdot)$ exactly reverses the forward flow, proving that flow matching only requires training a single vector field to achieve bidirectional view completion—unlike diffusion models that need two separate models for forward and reverse processes.

\section{Mathematical Proofs for Centered Log-Ratio (clr) Transformation}
\label{sec:clr-proof}

In this section, we formally prove that the centered logratio (clr) transformation maps the $D$-dimensional probability simplex onto the $(D-1)$-dimensional linear subspace of $\mathbb{R}^D$ defined by the zero-sum constraint. This ensures that clr-transformed data preserves the relative information of the original compositional distributions while lying in a linear space, which allows standard fusion methods to be applied validly without violating simplex consistency. This complements the discussion in \secref{sec:entropic-alignment}.

\subsection{Notation and Definitions}

Let $\mathcal{S}^D$ be the $D$-part simplex defined as:
\begin{equation}
    \mathcal{S}^D = \left\{ \mathbf{x} = [x_1, x_2, \dots, x_D] \in \mathbb{R}^D \;\middle|\; x_i > 0, \sum_{i=1}^D x_i = \kappa \right\},
\end{equation}
where $\kappa$ is a constant (typically $1$ or $100$).

The \textbf{geometric mean} of a composition $\mathbf{x}$ is denoted by $g(\mathbf{x})$:
\begin{equation}
    g(\mathbf{x}) = \left( \prod_{i=1}^D x_i \right)^{1/D}.
\end{equation}

The \textbf{Centered Log-Ratio (clr)} transformation is defined as the mapping $\text{clr}: \mathcal{S}^D \to \mathbb{R}^D$:
\begin{equation}
    \text{clr}(\mathbf{x}) = \mathbf{z} = \left[ \ln\frac{x_1}{g(\mathbf{x})}, \ln\frac{x_2}{g(\mathbf{x})}, \dots, \ln\frac{x_D}{g(\mathbf{x})} \right],
\end{equation}
where the $i$-th component is $z_i = \ln x_i - \ln g(\mathbf{x})$.

\subsection{Proposition 1: The Image of clr is a Linear Subspace}

\begin{proposition}
The image of the simplex under the clr transformation, denoted as $\mathcal{V}_0 = \{ \mathbf{z} \in \mathbb{R}^D \mid \exists \mathbf{x} \in \mathcal{S}^D, \mathbf{z} = \text{clr}(\mathbf{x}) \}$, is a linear subspace of $\mathbb{R}^D$ of dimension $D-1$. Specifically, it is the hyperplane defined by the sum-to-zero constraint:
\begin{equation}
    \mathcal{V}_0 = \left\{ \mathbf{z} \in \mathbb{R}^D \;\middle|\; \sum_{i=1}^D z_i = 0 \right\}.
\end{equation}
\end{proposition}

\begin{proof} This proof consists of two steps, as follows.

\paragraph{Step 1: Verify the sum-to-zero constraint.}
For any $\mathbf{z} = \text{clr}(\mathbf{x})$, the sum of its components is:
\begin{align}
    \sum_{i=1}^D z_i &= \sum_{i=1}^D \left( \ln x_i - \ln g(\mathbf{x}) \right) \nonumber \\
    &= \sum_{i=1}^D \ln x_i - \sum_{i=1}^D \ln g(\mathbf{x}) \nonumber \\
    &= \sum_{i=1}^D \ln x_i - D \cdot \ln g(\mathbf{x}).
\end{align}
Substituting the definition $g(\mathbf{x}) = (\prod_{j=1}^D x_j)^{1/D}$:
\begin{equation}
    D \cdot \ln g(\mathbf{x}) = D \cdot \ln \left( \prod_{j=1}^D x_j \right)^{1/D} = D \cdot \frac{1}{D} \sum_{j=1}^D \ln x_j = \sum_{j=1}^D \ln x_j.
\end{equation}
Therefore:
\begin{equation}
    \sum_{i=1}^D z_i = \sum_{i=1}^D \ln x_i - \sum_{i=1}^D \ln x_i = 0.
\end{equation}
Thus, $\mathcal{V}_0 \subseteq \{ \mathbf{z} \mid \sum z_i = 0 \}$.

\paragraph{Step 2: Verify Linear Subspace Axioms.}
Let $W = \{ \mathbf{z} \in \mathbb{R}^D \mid \sum_{i=1}^D z_i = 0 \}$. To prove $W$ is a linear subspace, we must show it contains the zero vector and is closed under addition and scalar multiplication.

\begin{enumerate}
    \item \textit{Zero Vector:} The vector $\mathbf{0} = [0, \dots, 0]$ satisfies $\sum 0 = 0$. Thus, $\mathbf{0} \in W$.
    
    \item \textit{Closure under Addition:} Let $\mathbf{u}, \mathbf{v} \in W$. Then $\sum u_i = 0$ and $\sum v_i = 0$.
    Consider $\mathbf{w} = \mathbf{u} + \mathbf{v}$. The sum of components is:
    \begin{equation}
        \sum_{i=1}^D w_i = \sum_{i=1}^D (u_i + v_i) = \sum_{i=1}^D u_i + \sum_{i=1}^D v_i = 0 + 0 = 0.
    \end{equation}
    Thus, $\mathbf{u} + \mathbf{v} \in W$.
    
    \item \textit{Closure under Scalar Multiplication:} Let $\mathbf{u} \in W$ and $c \in \mathbb{R}$.
    Consider $\mathbf{v} = c\mathbf{u}$. The sum of components is:
    \begin{equation}
        \sum_{i=1}^D v_i = \sum_{i=1}^D (c \cdot u_i) = c \sum_{i=1}^D u_i = c \cdot 0 = 0.
    \end{equation}
    Thus, $c\mathbf{u} \in W$.
\end{enumerate}

Since $W$ satisfies all axioms, it is a linear subspace of $\mathbb{R}^D$. Since the constraint $\sum z_i = 0$ is a single linearly independent equation in $\mathbb{R}^D$, the dimension of this subspace is $D-1$.
\end{proof}

\section{Experimental Supplementary Materials}\label{sec:experimental-supplement}

\subsection{Detailed dataset statistics and implementation specifics}
\label{sec:appendix-implementation-details}

\subsubsection{Implementation Details}
The experiments are conducted on one NVIDIA RTX4090 using PyTorch. 
We refer readers to Table~\ref{tab:experimental_setup} for detailed configurations of our experiments.

\begin{table*}[t]
    \centering
    \caption{\textbf{Experimental Configurations.} $\tau$ indicates  the view missing rate. Epochs represents the number of training epochs. Step indicates the number of flow matching steps. $\lambda_1, \lambda_2, \lambda_3, \lambda_4$ denote the hyperparameters in \secref{sec:analysis}.}
    \label{tab:experimental_setup}
    \resizebox{\linewidth}{!}{
    \setlength{\tabcolsep}{6pt} 
    \begin{tabular}{lccccc}
        \toprule
        \multirow{2}{*}{\textbf{Setting}} & \multicolumn{5}{c}{\textbf{Dataset}} \\
        \cmidrule{2-6}
        & \textbf{Synthetic3d} & \textbf{CUB} & \textbf{HandWritten} & \textbf{LandUse-21} & \textbf{Fashion} \\  
        \midrule
        \multicolumn{6}{l}{\textit{Shared configurations across all datasets}}\\
        \midrule
        Optimizer & \multicolumn{5}{c}{Adam} \\
        Learning Rate & \multicolumn{5}{c}{$1 \times 10^{-3}$} \\ 
        Epochs & \multicolumn{5}{c}{200} \\
        $\tau$ & \multicolumn{5}{c}{$\{0.1, 0.3, 0.5\}$} \\  
        Step & \multicolumn{5}{c}{1 (train and inference)} \\
        \midrule
        \multicolumn{6}{l}{\textit{Individual configurations per dataset}}\\
        \midrule
        Batch Size & 128 & 256 & 256 & 128 & 128 \\
        $\lambda_1$ & 100.0 & 1.0 & 1000.0 & 1000.0 & 10.0 \\
        $\lambda_2$ & 1.0 & 10.0 & 1.0 & 1.0 & 1.0 \\
        $\lambda_3$ & 1.0 & 1.0 & 1.0 & 1.0 & 1.0 \\
        $\lambda_4$ & 1.0 & 1.0 & 1.0 & 1.0 & 1.0 \\
        \bottomrule
    \end{tabular}
    }
\end{table*}

\subsection{Comparison of Different Strategies}

\begin{table}[t]
\centering
\caption{\textbf{Comparison of different strategies on CUB under different missing rate.} \textbf{ACC}, \textbf{NMI}, and \textbf{ARI} denote clustering Accuracy, Normalized Mutual Information, and Adjusted Rand Index, respectively (higher is better).
$\tau$ indicates the missing rate of views.}
\label{tab:comparison}
\resizebox{\linewidth}{!}{
\begin{tabular}{lccccccccc}
\toprule
\multirow{2}{*}{Strategies} & \multicolumn{3}{c}{$\tau = 0.1$} & \multicolumn{3}{c}{$\tau = 0.3$} & \multicolumn{3}{c}{$\tau = 0.5$} \\
\cmidrule{2-10}
                            & ACC    & NMI    & ARI    & ACC    & NMI    & ARI    & ACC    & NMI    & ARI    \\
\midrule
GAN                         & 60.25  & 58.72  & 52.64  & 55.82  & 53.43  & 48.39  & 51.36  & 49.87  & 44.52  \\
Prediction                  & 75.89  & 71.26  & 59.83  & 71.13  & 66.76  & 54.61  & 67.54  & 63.29  & 50.17  \\
Paired                      & 77.41  & 71.53  & 59.27  & 72.83  & 66.80  & 53.98  & 68.92  & 63.55  & 49.86  \\
Diffusion                   & 82.23  & 77.70  & 69.21  & 77.17  & 71.35  & 59.85  & 75.50  & \textbf{72.21}  & 59.12  \\
\midrule
Flow(Ours)                  & \textbf{86.00} & \textbf{78.19} & \textbf{72.28} & \textbf{82.67} & \textbf{72.80} & \textbf{66.31} & \textbf{80.50} & 71.74 & \textbf{63.15} \\
\bottomrule
\end{tabular}
}
\end{table}

To further evaluate the effectiveness of our flow-based generation strategy, we replaced it with alternative generative models, including GAN \cite{wang2020generative}, Prediction \cite{lin2021completer}, Paired-data-only training, and Diffusion models \cite{zhang2025incomplete}, and conducted comprehensive experiments on the CUB dataset with missing rates  of 0.1, 0.3 and 0.5, with corresponding results shown in Table~\ref{tab:comparison}.

The results consistently demonstrate that our Flow-based strategy achieves dominant performance over all comparative methods across all missing rate settings, verifying the robust and effective nature of our proposed method. Specifically, compared with the Diffusion baseline that ranks second, our method yields substantial and notable improvements in all clustering metrics under each missing rate, which fully reflects the superior generation capability of the flow-based strategy for missing view recovery. Additionally, the comparison with the ``Paired'' strategy confirms that our method can fully and efficiently leverage incomplete data instances instead of simply discarding them, which maximizes the utility of multi-view information and thus stably boosts the clustering performance even with the increase of missing rates, showing better adaptability to different missing rate scenarios.

\subsection{Visualization Analysis.}

We visualize the learned representations on the CUB, LandUse-21, Synthetic3d and Fashion datasets using t-SNE under the missing rate $\tau=0.3$, as shown in Fig.~\ref{fig:appendix-visualizations}. 
Compared with scattered and mixed raw features, the representations learned by our framework form well-separated clusters. This improvement stems from the cluster-level contrastive loss and entropy-based alignment, which enforce cross-view assignment consistency and stabilize probability predictions. Meanwhile, the flow-matching module achieves cross-view completion by performing straight-path transport between paired latent representations.

\begin{figure}[t]
    \centering 
    \begin{subfigure}[b]{0.45\textwidth}
        \centering
        \includegraphics[width=\textwidth]{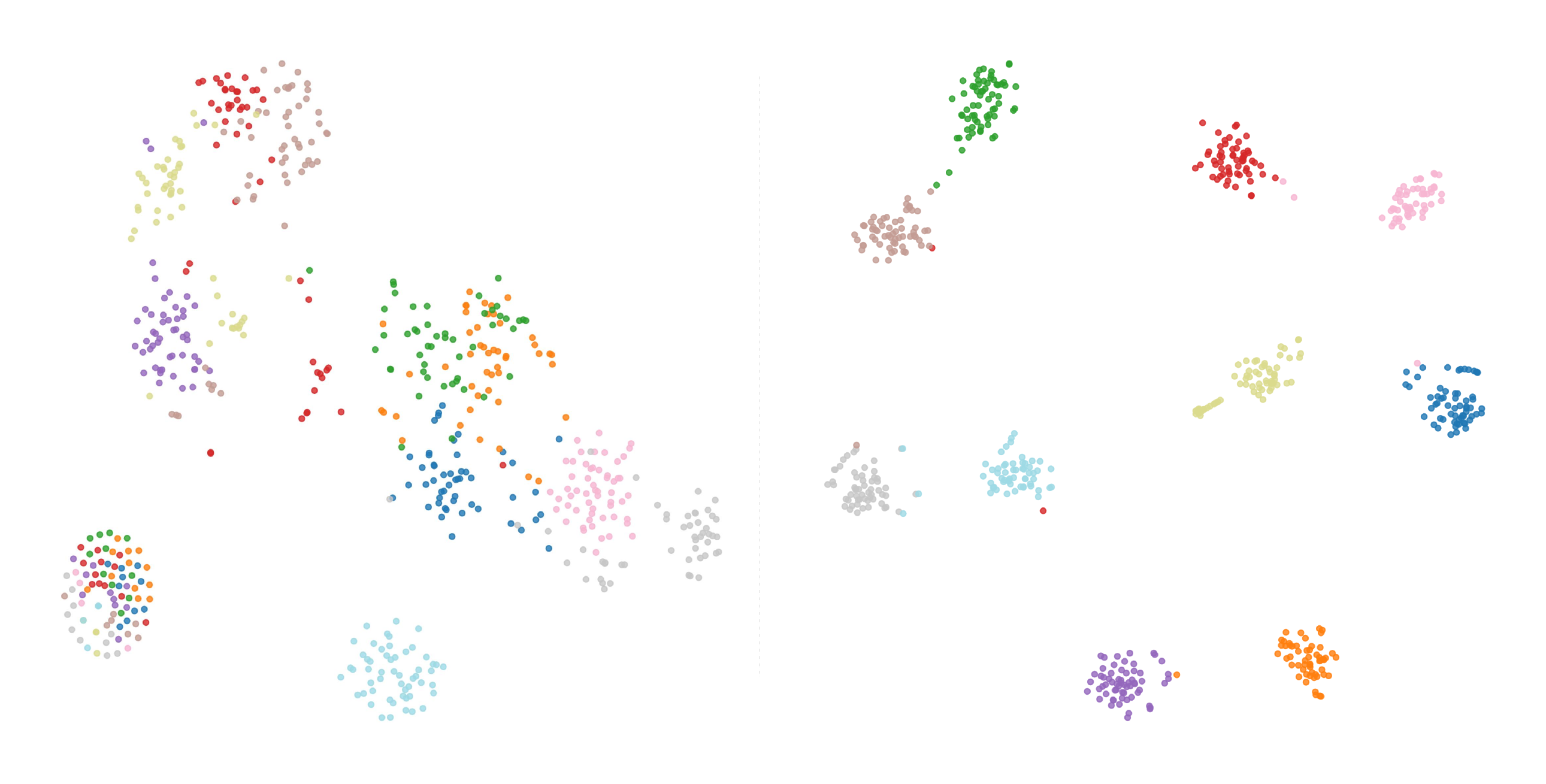} 
        \caption{CUB} 
        \label{subfig:cub}          
    \end{subfigure}
    \hfill  
    \begin{subfigure}[b]{0.45\textwidth}
        \centering
        \includegraphics[width=\textwidth]{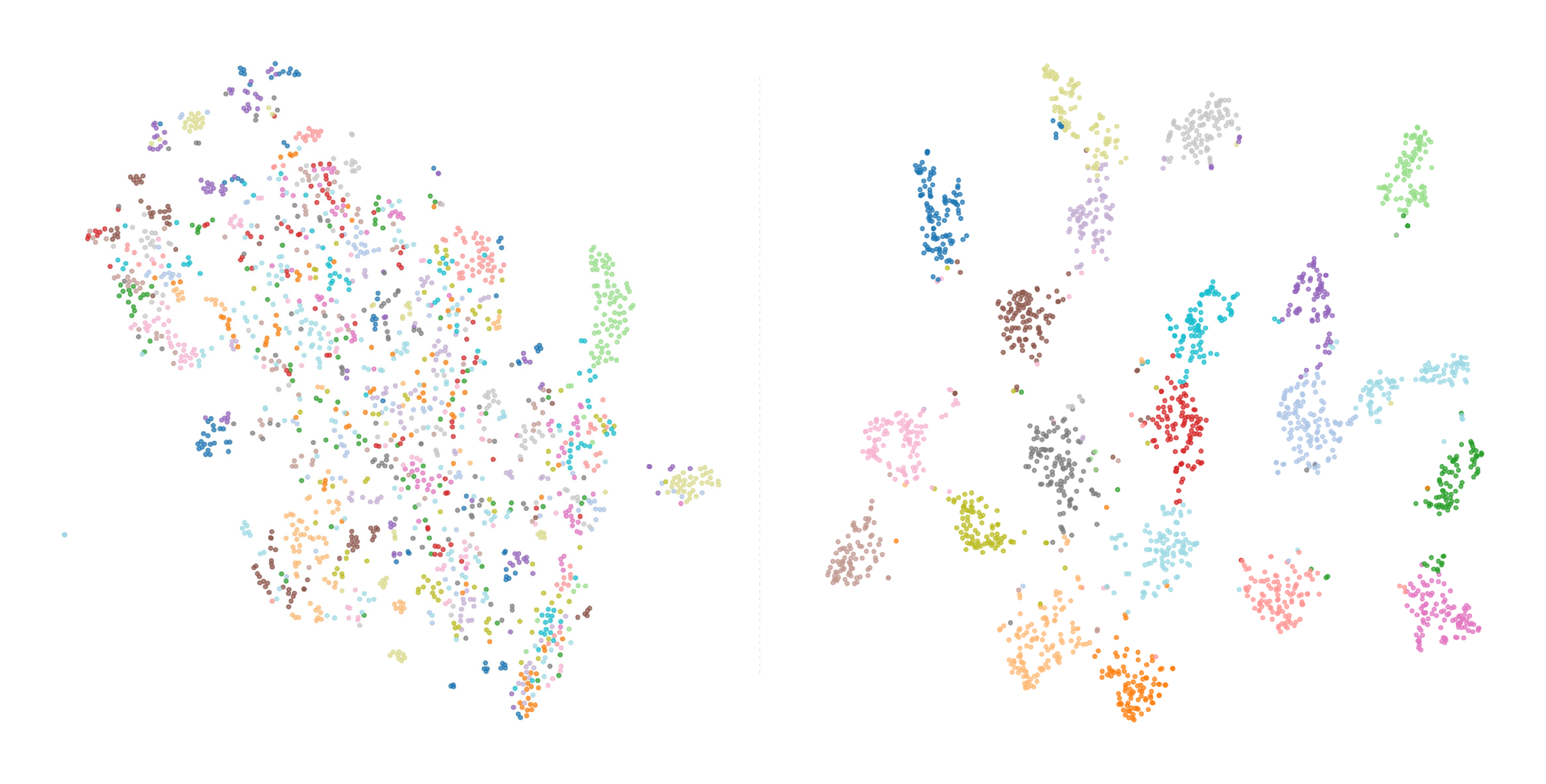}
        \caption{LandUse-21}
        \label{subfig:landuse}
    \end{subfigure}
    \hfill
    \begin{subfigure}[b]{0.45\textwidth}
        \centering
        \includegraphics[width=\textwidth]{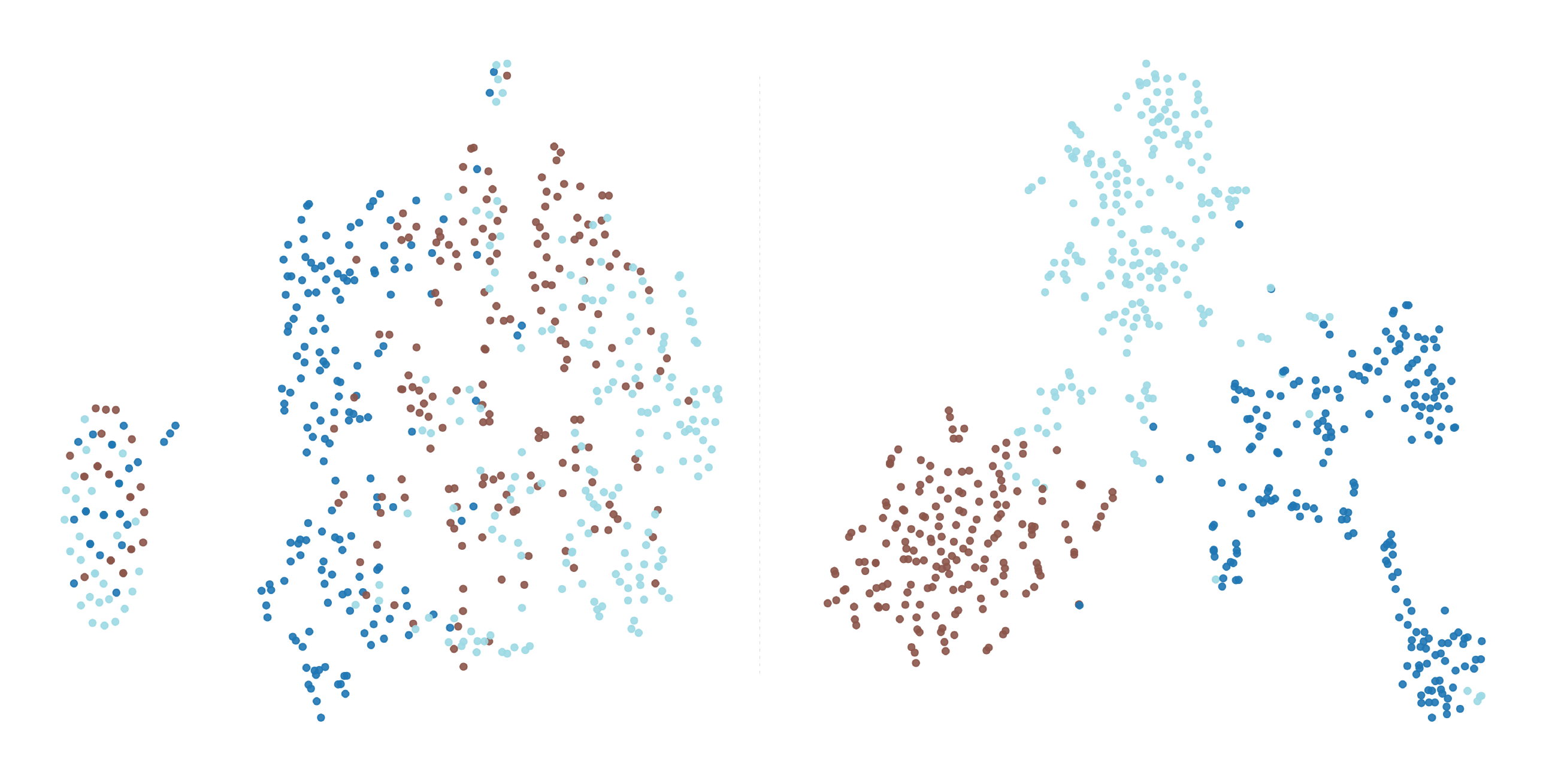}  
        \caption{Synthetic3d}  
        \label{subfig:synthetic3d}         
    \end{subfigure}
    \hfill  
    \begin{subfigure}[b]{0.45\textwidth}
        \centering
        \includegraphics[width=\textwidth]{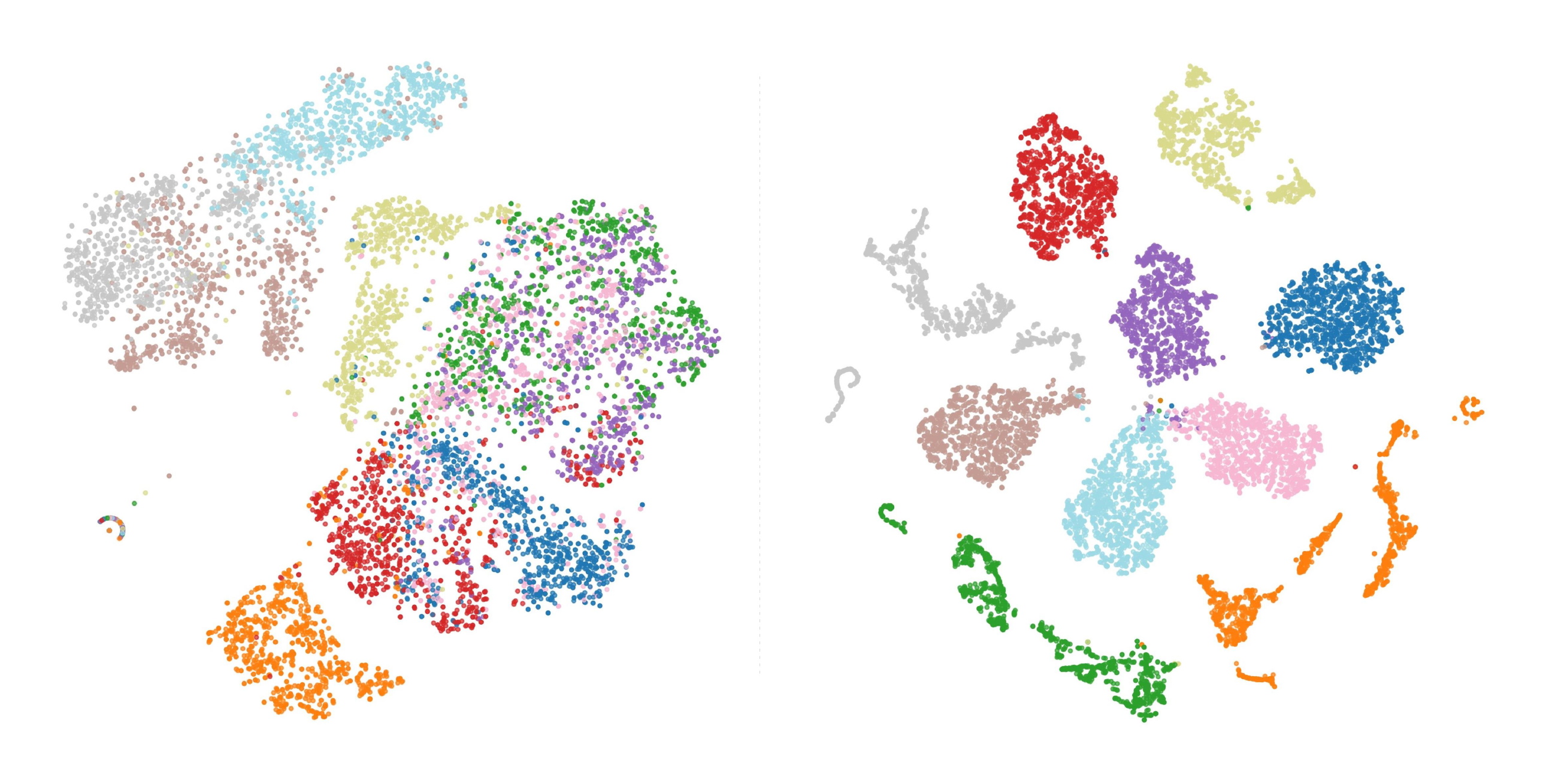}
        \caption{Fashion}
        \label{subfig:fashion}
    \end{subfigure}
    \caption{\textbf{Representation visualization.} (a), (b), (c), and (d) show the t-SNE visualizations of CUB, LandUse-21, Synthetic3d, and Fashion datasets, respectively, with a missing rate of 0.3.
}
    \label{fig:appendix-visualizations}
\end{figure}

\end{document}